\documentclass{article}

\usepackage{graphicx}
\usepackage{amssymb,amsmath,color}

\usepackage{hyperref}

  \hypersetup{colorlinks,
            linkcolor=blue,
            citecolor=blue,
            urlcolor=magenta,
            linktocpage,
            plainpages=false}

\RequirePackage{natbib}

\bibpunct{(}{)}{;}{a}{,}{,}

\usepackage[mathcal]{eucal}

\def \E{{\mathbb E}}

\newcommand{\BEAS}{\begin{eqnarray*}}
\newcommand{\EEAS}{\end{eqnarray*}}
\newcommand{\BEA}{\begin{eqnarray}}
\newcommand{\EEA}{\end{eqnarray}}
\newcommand{\BEQ}{\begin{equation}}
\newcommand{\EEQ}{\end{equation}}
\newcommand{\BIT}{\begin{itemize}}
\newcommand{\EIT}{\end{itemize}}
\newcommand{\BNUM}{\begin{enumerate}}
\newcommand{\ENUM}{\end{enumerate}}
\newcommand{\BA}{\begin{array}}
\newcommand{\EA}{\end{array}}

\newcommand{\tr}{\mathop{ \rm tr}}

\newcommand{\idm}{I}
\newcommand{\rb}{\mathbb{R}}

\newcommand{\BlackBox}{\rule{1.5ex}{1.5ex}}  % end of proof

\newcommand{\mysec}[1]{Section~\ref{sec:#1}}
\newcommand{\eq}[1]{Eq.~(\ref{eq:#1})}
\newcommand{\myfig}[1]{Figure~\ref{fig:#1}}

\newenvironment{proof}{\par\noindent{\bf Proof\ }}{\hfill\BlackBox\\[2mm]}
\newtheorem{lemma}{Lemma}

\newtheorem{proposition}{Proposition}

\title{On the Effectiveness of Richardson Extrapolation\\ in Machine Learning}
\author{Francis Bach  \\
INRIA - D\'epartement d'Informatique de l'Ecole Normale Sup\'erieure \\
PSL Research University
Paris, France \\
\texttt{francis.bach@inria.fr}}

\date{\today}

\begin{document}

 \maketitle
  
\begin{abstract}
Richardson extrapolation is a classical technique from numerical analysis that can improve the approximation error of an estimation method by combining linearly several estimates obtained from different values of one of its  hyperparameters, without the need to know in details the inner structure of the original estimation method. The main goal of this paper is to study when Richardson extrapolation can be used within machine learning, beyond the existing applications to step-size adaptations in stochastic gradient descent. We identify two situations where Richardson interpolation can be useful: (1) when the hyperparameter is  the number of iterations of an existing iterative optimization algorithm, with applications to  averaged gradient descent and Frank-Wolfe algorithms (where we obtain asymptotically rates of $O(1/k^2)$ on polytopes, where $k$ is the number of iterations), and (2) when it is a regularization parameter, with applications to Nesterov smoothing techniques for minimizing non-smooth functions (where we obtain asymptotically rates close to $O(1/k^2)$ for non-smooth functions), and ridge regression. In all these cases, we show that extrapolation techniques come with no significant loss in performance, but with sometimes strong gains, and we  provide theoretical justifications based on asymptotic developments for such gains, as well as empirical illustrations on classical problems from machine learning. \end{abstract}

\section{Introduction}

Many machine learning methods can be cast as looking for approximations of some ideal quantity which cannot be readily computed from the data at hand: this ideal quantity can be the predictor learned from infinite data, or an iterative algorithm run for infinitely many iterations.  Taking theirs roots in optimization and more generally numerical analysis, many accelerations techniques have been developed to tighten these approximations with as few changes as possible to the original method.

While some acceleration techniques add some simple modifications to a known algorithm, such as Nesterov acceleration for the gradient descent method~\citep{nesterov1983method}, \emph{extrapolation} techniques are techniques that do not need to know the fine inner structure of the method to be accelerated. These methods are only based on the observations of solutions of the original method.

Extrapolation techniques work on the vector-valued output $x_t\in \rb^d$ of the original method that depends on some controllable real-valued quantity $t$, which can be the number of iterations or some regularization parameter, and more generally any parameter that controls both the running time and the approximation error of the algorithm.  When $t$ tends to~$t_\infty$ (which is typically $0$ or $+\infty$), we will assume that $x_t$ has an asymptotic expansion of the form 
$$x_t= x_\ast + g_t+ O(h_t),$$
where $x_\ast$ is the desired output,  $g_t \in \rb^d$ is the asymptotic equivalent of $x_t - x_\ast$, and $h_t = o( \| g_t\|)$. The key question in extrapolation is the following:  from the knowledge of $x_t$ for potentially several $t$'s, how can we better approximate~$x_\ast$, \emph{without the full knowledge} of $g_t$? 

For exponentially converging algorithms, there exist several ``non-linear'' schemes that combine linearly several values of $x_t$ with weights that depend non-linearly on the iterates, such as Aitken's $\Delta^2$ process~\citep{aitken1927xxv} or  Anderson acceleration~\citep{anderson1965iterative}, which has recently been shown to provide significant acceleration to linearly convergent gradient-based algorithms~\citep{scieur2016regularized}. In this paper, we consider dependence in powers of~$t$, where Richardson extrapolation excels~\citep[see, e.g.,][]{richardson1911ix,joyce1971survey,gautschi1997numerical}.

\begin{figure}
\begin{center}
\includegraphics[scale=1.2]{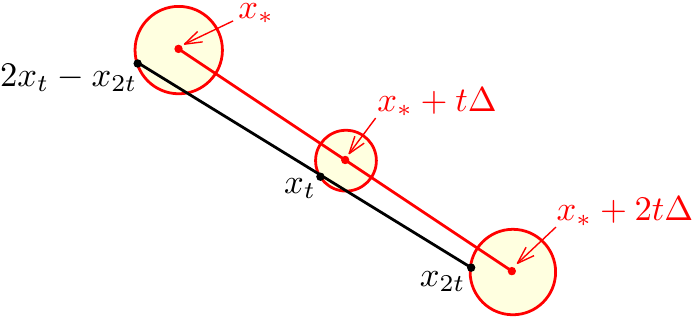}  

\vspace*{-.25cm}

\caption{Illustration of Richardson extrapolation for $t_\infty = 0$ and $x_t = x_\ast + t \Delta + O(t^2)$. Iterates (in black) with their first-order expansions (in red). The deviations (represented by circles) are of order $O(t^2)$. Adapted from~\citet{dieuleveut2017bridging}. \label{fig:richardson}}
\end{center}
\end{figure}

We thus assume that 
$$g_t =  t^\alpha \cdot \Delta, $$
and $h_t = t^\beta$ is a power of $t$ such that $h_t= o(\| g_t\| )$, where $\alpha \in \rb$ is known but $\Delta \in \rb^d$ is unknown, that is
$$
x_t= x_\ast + t^\alpha \Delta+ O(t^\beta).
$$
In all our cases, $\alpha=-1$ when $t_\infty=+\infty$ and $\alpha =1$ when $t=0$. Richardson extrapolation is simply combining two iterates with different values of $t$ so that the zero-th order term $x_\ast$ is preserved, while the first-order term cancels, for example:
\BEAS
& & 2 x_t - x_{2^{1/\alpha} t } \\
& \!\!= \!\!& 2 ( x_\ast + t^\alpha \Delta  + O(t^\beta)  ) - ( x_\ast +  2t^\alpha \Delta  + O(t^\beta) ) \\
&\!\! =\!\! & x_\ast + O(t^\beta).
\EEAS
See an illustration in \myfig{richardson} for $\alpha=1$, $\beta=2$, and $t_\infty=0$. Note that: (a) the choice of $2^{1/\alpha} \neq 1 $ as a multiplicative factor is arbitrary and chosen for its simplicity when $|\alpha|=1$, and   (b) Richardson extrapolation can be used with $m+1$ iterates to remove the first $m$ terms in a asymptotic expansion, where the powers of the expansion are known and not the associated vector-valued constants (see examples in \mysec{regularization}).
 
The main goal of this paper is to study when Richardson extrapolation can be used within machine learning, beyond the existing explicit applications to step-size adaptations in stochastic gradient descent~\citep{durmus2016stochastic,dieuleveut2017bridging}, and implicit wide-spread use in ``tail-averaging'' (see more details in \mysec{tail}).

We identify two situations where Richardson interpolation can be useful:
\BIT
\item $t=k$ is the number of iterations of an existing iterative optimization algorithm converging to $x_\ast$, where then $\alpha = -1$ and $t_\infty=+\infty$, and Richardson extrapolation considers, for $k$ even, $x_k^{(1)} = 2x_k - x_{k/2}$. We consider in \mysec{iteration}, averaged gradient descent and Frank-Wolfe algorithms (where we obtain asymptotically rates of $O(1/k^2)$ on polytopes, where $k$ is the number of iterations).

\item $t=\lambda$ is a regularization parameter, where then $\alpha=1$ and $t_\infty = 0$, and Richardson extrapolation considers $x^{(1)}_\lambda = 2x_\lambda - x_{2\lambda}$. We consider in \mysec{regularization}, Nesterov smoothing techniques for minimizing non-smooth functions (where we obtain asymptotically rates close to $O(1/k^2)$ for non-smooth functions), and ridge regression (where we obtain estimators with lower bias).

\EIT
As we will show, extrapolation techniques come with no significant loss in performance, but with sometimes strong gains, and the goal of this paper is to provide theoretical justifications for such gains, as well as empirical illustrations on classical problems from machine learning. Note that  we aim for the simplest asymptotic results (most can be made non-asymptotic with extra assumptions).

\section{Extrapolation on the  Number of Iterations}
\label{sec:iteration}

In this section, we consider extrapolation based on the number of iterations $k$, that is, for the simplest case 
$$x_k^{(1)} = 2x_k - x_{k/2},$$ for optimization algorithms aimed at minimizing  on $\rb^d$
If $x_k$ is converging to a minimizer $x_\ast$, then so is $x_{k/2}$, and thus also $x_k^{(1)} = 2x_k - x_{k/2}$; moreover, we have
$\|x_k^{(1)} - x_\ast\|_2 \leqslant 2 \| x_k - x_\ast\|_2 + \| x_{k/2} - x_\ast\|_2$, so even if there are no cancellations,  performance is never significantly  deteriorated (the risk is essentially to lose half of the iterations). 

The potential gains depend  on the way $x_k$ converges to $x_\ast$. The existence of a convergence rate of the form $f(x_k) -f(x_\ast) = O(1/k)$  or $O(1/k^2)$ is not enough, as Richardson extrapolation requires a specific direction of asymptotic convergence. As illustrated in \myfig{oscillating}, some algorithms are oscillating around their solutions, while some converge with a specific direction. Only the latter ones can be accelerated with Richardson extrapolation, while the former ones are good candidates for Anderson acceleration~\citep{anderson1965iterative,scieur2016regularized}.

We now consider three algorithms: (1) averaged gradient descent, where extrapolation is at its best, as it transforms an $O(1/t^2)$ convergence rate into an exponential one, (2) accelerated gradient descent, where extrapolation does not bring anything, and (3) Frank-Wolfe algorithms, where the situation is mixed (sometimes it helps, sometimes it does not).

\begin{figure}
\begin{center}
\includegraphics[scale=1.2]{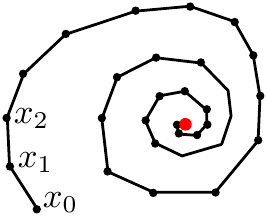} \hspace*{1cm}
\includegraphics[scale=1.2]{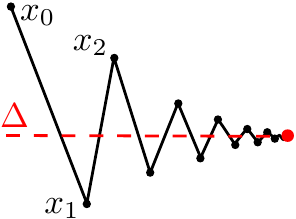}

\vspace*{-.2cm}

\caption{Left: Oscillating convergence, where Richardson extrapolation does not lead to any gain. Right: non-oscillating  convergence, with a main direction $\Delta$ (in red dotted), where Richardson extrapolation can be beneficial if the oscillations orthogonal to the direction $\Delta$ are negligible compared to convergence along the direction~$\Delta$. \label{fig:oscillating}}
\end{center}
\end{figure}

\subsection{Averaged gradient descent}
\label{sec:tail}
We consider the usual gradient descent algorithm
$$x_k = x_{k-1} - \gamma f'(x_{k-1}),$$
where $\gamma \geqslant 0$ is a step-size, 
with Polyak-Ruppert averaging~\citep{polyak1992acceleration,ruppert}:
$$
\bar{x}_k = \frac{1}{k} \sum_{i=0}^{k-1} x_i.
$$
Averaging is key to robustness to potential noise of the gradients~\citep{polyak1992acceleration,nemirovski2009robust}. However it comes with the unintended consequence of losing the exponential forgetting of initial conditions for strongly convex problems~\citep{moulines2011non}.

A common way to restore exponential convergence (up to the noise level in the stochastic case) is to consider ``tail-averaging'', that is, to replace $\bar{x}_k$ by the average of only the latest $k/2$ iterates~\citep{JMLR:v18:16-595}. As shown below for $k$ even, this corresponds exactly to Richardson extrapolation:
$$
\frac{2}{k} \sum_{i=k/2}^{k-1} x_i
= 
\frac{2}{k} \sum_{i=0}^{k-1} x_i - \frac{2}{k} \sum_{i=0}^{k/2-1} x_i 
=
2 \bar{x}_k - \bar{x}_{k/2}.
$$
While \citet{JMLR:v18:16-595} focuses on a non-asymptotic analysis for stochastic problems for least-squares regression, we now provide an asymptotic analysis for general convex objective functions and non-stochastic problems (see proof in Appendix~\ref{app:gd}).

\begin{proposition}
\label{prop:gd}
Assume   $f$  convex, three-times differentiable with Hessian eigenvalues bounded by~$L$, bounded third-order derivatives, and a unique minimizer $x_\ast \in \rb^d$ such that $f''(x_\ast) $ is positive definite. If $\gamma \leqslant 1/L$, then 
$$
\bar{x}_k = x_\ast + \frac{1}{k} \Delta + O(\exp(- k \lambda ) ),
$$
where  $\Delta = \sum_{i=0}^\infty (x_i - x_\ast)$ and $\lambda$ is proportional to  $ \gamma\lambda_{\min}(f''(x_\ast))$.
\end{proposition}

 Note that:
(a) before Richardson extrapolation, the asymptotic convergence rate will be of the order $O(1/k^2)$, which is better than the usual $O(1/k)$ upper-bound for the rate of gradient descent, but with a stronger assumption that  in fact leads to exponential convergence before averaging,   (b) while $\Delta$ has a simple expression, it cannot be computed in practice,  (c) that Richardson extrapolation leads to an exponentially convergent algorithm from an algorithm converging asymptotically in $O(1/k^2)$ for functions values, and (d) that in the presence of noise in the gradients, the exponential convergence would only be up to the noise level. See \myfig{gd} (left plot) for an illustration with noisy gradients.

\begin{figure}
\begin{center}
\hspace*{-.2cm} \includegraphics[scale=.4]{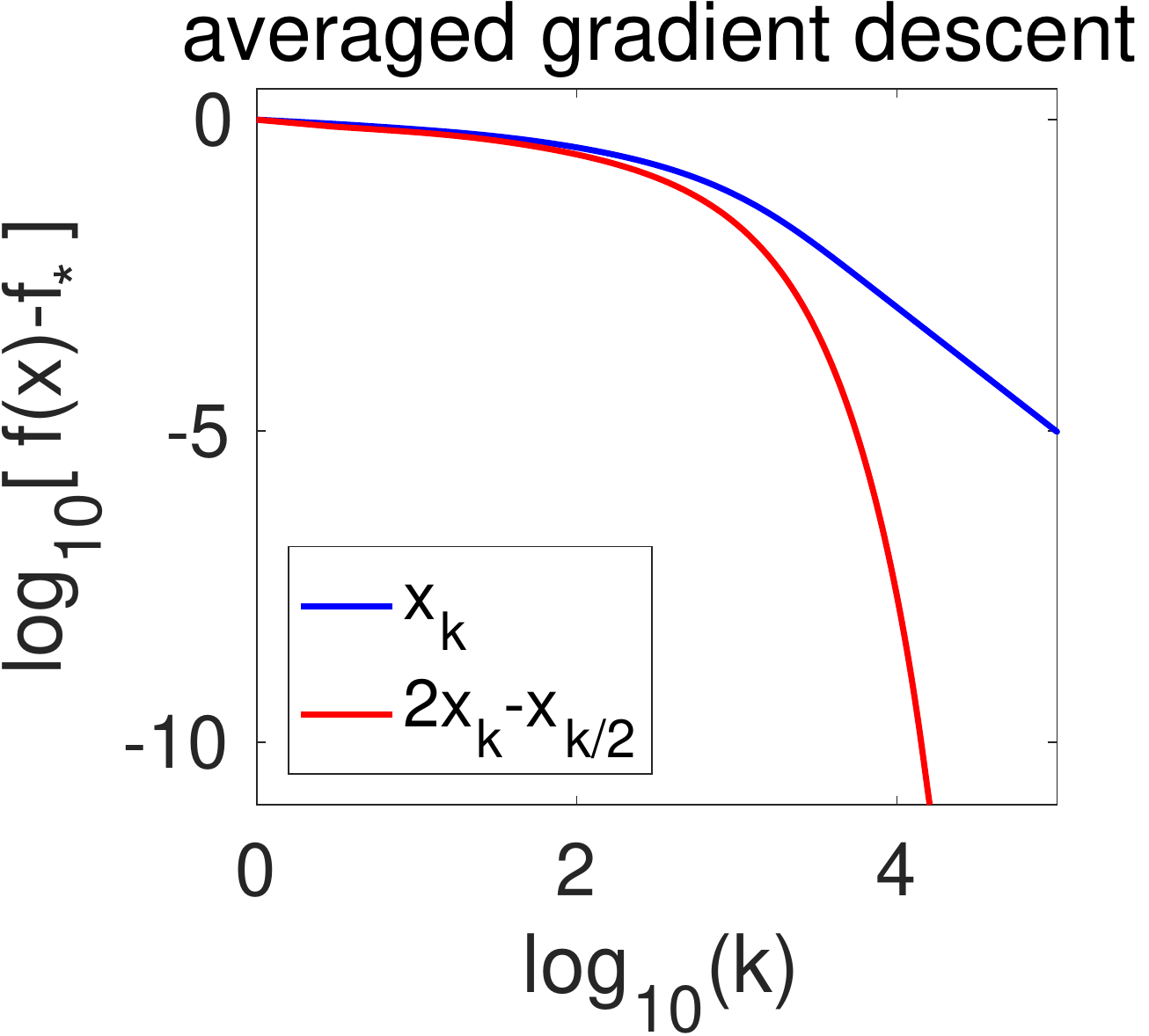} \hspace*{.1cm}
\includegraphics[scale=.4]{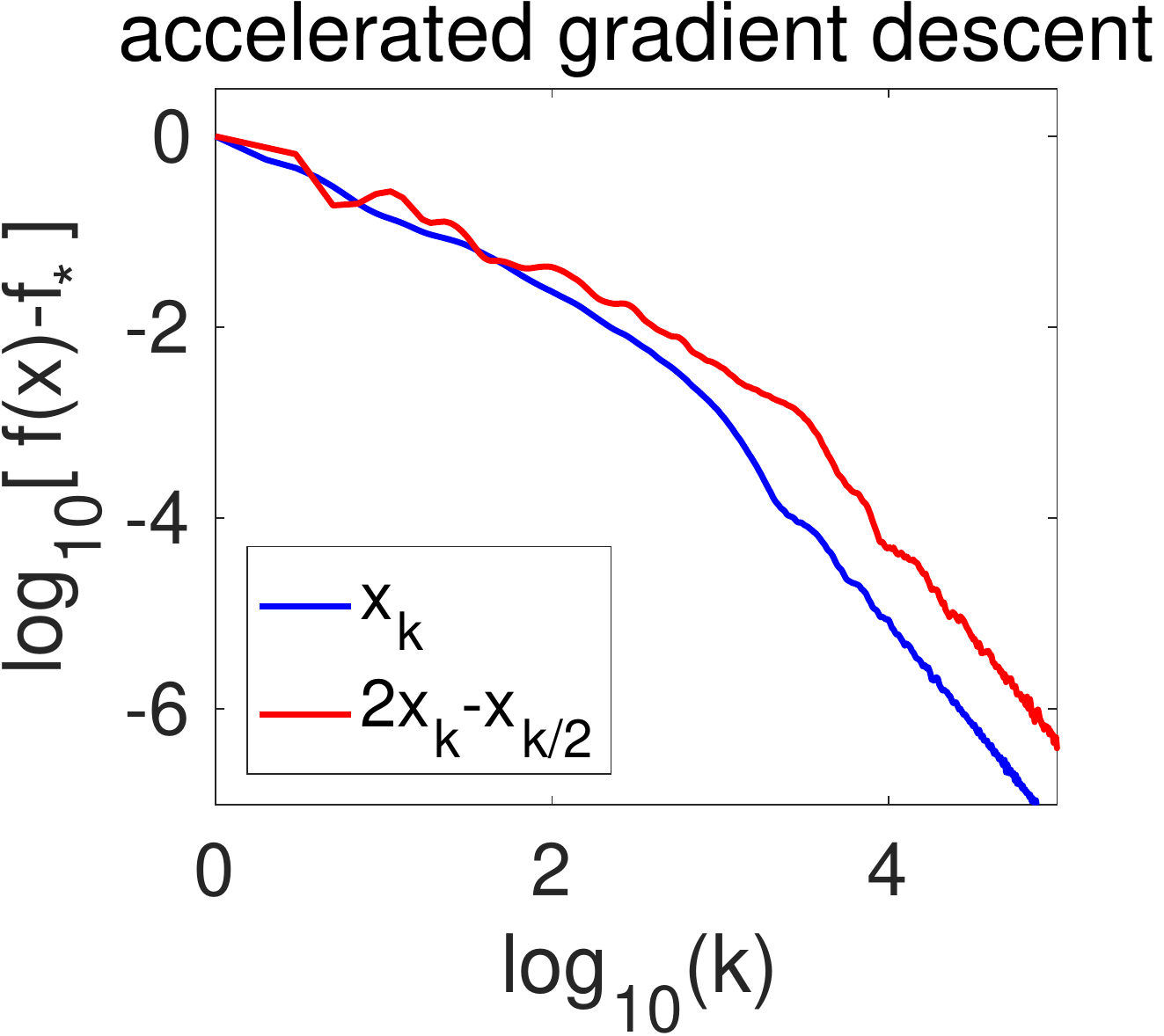} \hspace*{-.2cm} 

\vspace*{-.2cm}

\caption{Left: averaged gradient descent on a logistic regression problem in dimension $d=400$, and with $n=4000$ observations
$
\min_{x \in \rb^d} f(x) = \frac{1}{n} \sum_{i=1}^n \log( 1+ \exp( - b_i x^\top a_i))
$, with $(a_i,b_i) \in \rb^d \times \{-1,1\}$. The covariance matrix of (Gaussian) inputs has eigenvalues  $1/j$, $j =1,\dots,d$; the lowest eigenvalue is $1/400$ and therefore we can see the effect of strong convexity starting between $k=100$ and $1000$ iterations; moreover, for the regular averaged recursion, the line in the log-log plot has slope $-2$. Right: accelerated gradient descent on a quadratic optimization problem in dimension $d=1000$ and a Hessian whose eigenvalues are $1/j^2$, $j =1,\dots,d$; with such eigenvalues, the local linear convergence is not observed and we have a line of slope~$-2$.  \label{fig:gd}}
\end{center}
\end{figure}

\subsection{Accelerated gradient descent}
In the section above, we considered averaged gradient descent, which is asymptotically converging as $O(1/k^2)$ and on which Richardson extrapolation could be used with strong gains. Is it possible also for the accelerated gradient descent~\citep{nesterov1983method}, which has a (non-asymptotic) convergence rate of $O(1/k^2)$ for convex functions?

It turns out that the behavior of the iterates of accelerated gradient descent is exactly of the form depicted in the left plot of \myfig{oscillating}: that is, the iterates $x_k$ oscillate around the optimum, as can be seen from the spectral analysis for quadratic problems, in continuous time~\citep{su2016differential} or discrete time~\citep{flammarion2015averaging}. Richardson extrapolation is of no help, but is not degrading performance too much. See \myfig{gd} (right plot) for an illustration.

\subsection{Frank-Wolfe algorithms}
We now consider Frank-Wolfe algorithms (also known as conditional gradient algorithms) for minimizing a smooth convex function $f$ on a compact convex set $\mathcal{K}$. These algorithms are dedicated to situations where one can easily minimize linear functions on $\mathcal{K}$~\citep[see, e.g.,][and references therein]{jaggi2013revisiting}. The algorithm has the following form:
\BEAS
\bar{x}_k & \in & \arg\min_{x \in \mathcal{K}}\  f(x_{k-1}) + f'(x_{k-1})^\top (x -x_{k-1})  \\
x_k & = & ( 1- \rho_k) x_{k-1} + \rho_k \bar{x}_k.
\EEAS
That is, the first order Taylor expansion of $f$ at $x_{k-1}$ is minimized, ending up typically in an extreme point $\bar{x}_k$ or $\mathcal{K}$ and a convex combination of $x_{k-1}$ and $\bar{x}_k$ is considered. While some form of line search can be used to find $\rho_k$, we consider so-called ``open loop'' schemes where $\rho_k = 1/k$ or $\rho_k = 2/(k+1)$~\citep{dunn1978conditional,jaggi2013revisiting}.

In terms of function values, these two variants are known to converge at respective rates $O( \log(k) / k)$ and $O(1/k)$. Moreover, as illustrated in \myfig{FWzigzag}, they are known to zig-zag towards the optimal point. Avoiding this phenomenon can be done in several ways, for examples through optimizing over all convex combinations of the $\bar{x}_i$'s for $i \leqslant k$~\citep{von1977simplicial}, or through so-called ``away steps''~\citep{guelat1986some,lacoste2015global}. In this section, we consider Richardson extrapolation and assume for simplicity that $\mathcal{K}$ is a polytope (which is a typical use case for Frank-Wolfe algorithms). Note here that we are considering asymptotic convergence rates, and even without extrapolation (but with a local strong-convexity assumption), we can beat the $O(1/k)$ rates for the step-size $\rho_k = 2/(k+1)$.\footnote{Note that: (a) the lower bound with dependence $O(1/k)$ from~\citet{canon1968tight} only applies  to Frank-Wolfe algorithms with line-search, and (b) that our bounds are local and the constants have to depend on dimension so as to not contradict the lower bound from~\citet{jaggi2013revisiting}.}

\begin{figure}
\begin{center}
\includegraphics[scale=1.2]{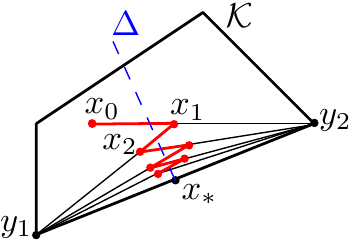}

\vspace*{-.2cm}

\caption{Frank-Wolfe algorithm zigzagging. Starting from $x_0$, the algorithm always moves towards one of the extreme points of~$\mathcal{K}$, with an average direction of $\Delta$.\label{fig:FWzigzag}   }
\end{center}
\end{figure}

\begin{figure}
\begin{center}
\hspace*{-.2cm}
\includegraphics[scale=.41]{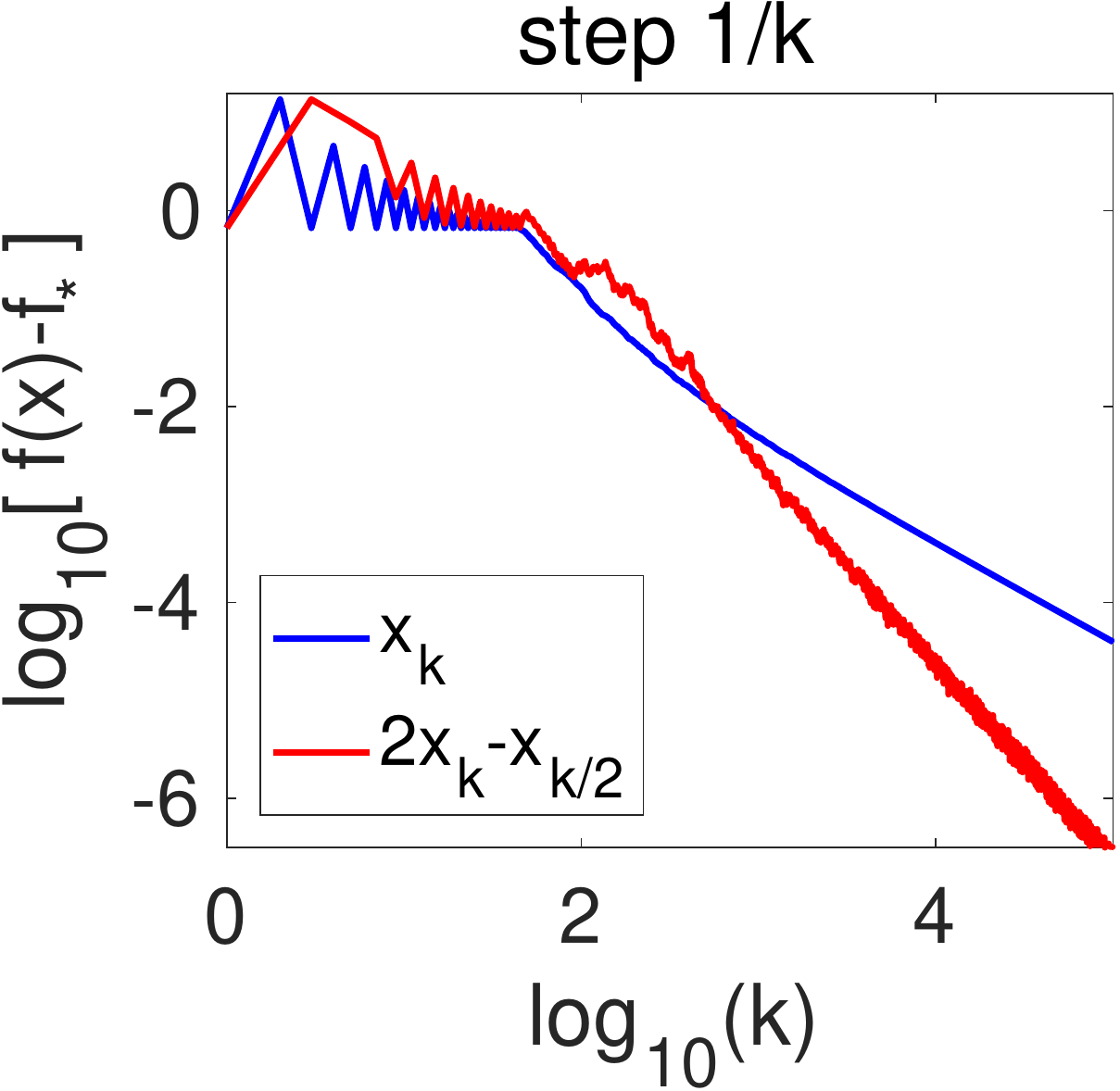}\hspace*{.4cm}
\includegraphics[scale=.41]{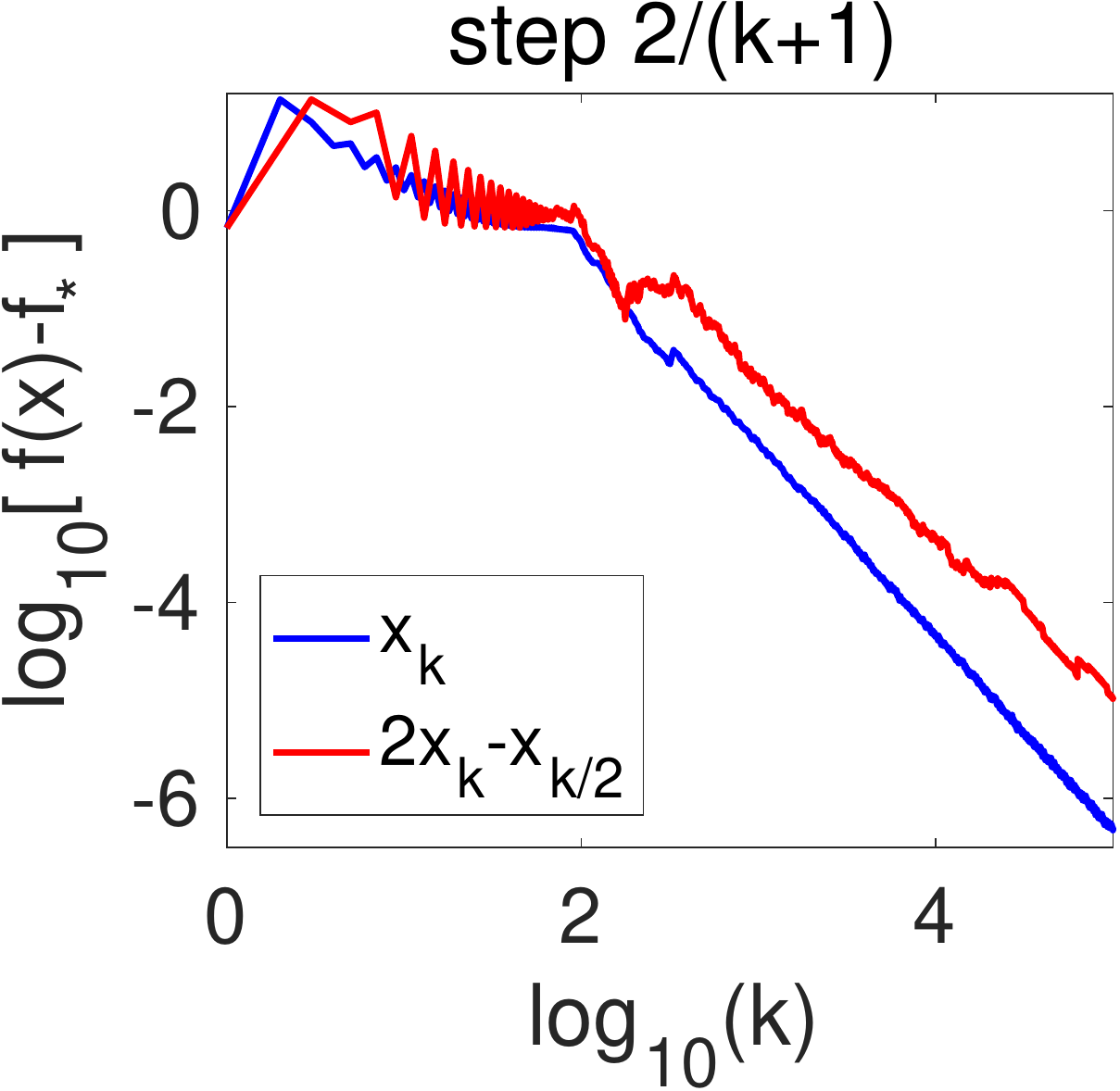}

\vspace*{-.2cm}

\caption{Frank-Wolfe for the ``constrained logistic Lasso'', that is, $
\min_{x \in \rb^d} f(x) \mbox{ such that} \| x\|_1 \leqslant c$, with $f(x) = \frac{1}{n} \sum_{i=1}^n \log( 1+ \exp( - b_i x^\top a_i))$. We consider   $n=400$ observations in dimension $d=400$, sampled from a standard normal distribution, and with a constraint on the $\ell_1$-norm. Left: step size $1/k$ with slopes~$-1$ (blue) and $-2$ (red). Right: step size $2/(k+1)$, with slope approximately~$-2$ for the two curves. \label{fig:FWlasso}  }
\end{center}
\end{figure}

\paragraph{Asymptotic expansion.} 
In order to provide the proposition below (see  Appendix~\ref{app:FW} for a proof) that  characterizes the zig-zagging phenomenon, we assume regularity properties similar to \mysec{tail} and that the unique minimizer is ``in the middle of a face'' of $\mathcal{K}$, which is often referred to as \emph{constraint qualification} in optimization~\citep{nocedal2006numerical}.  

\begin{proposition}
\label{prop:FW}
Assume   $f$  convex, three-times differentiable with bounded third-order derivatives in the polytope $\mathcal{K}$, and with a unique minimizer $x_\ast \in \rb^d$ such that $f''(x_\ast) $ is positive definite. Moreover, we assume $x_\ast$ is strictly in a $(m- 1)$-dimensional face of $\mathcal{K}$ which is the convex hull $\mathcal{K}_\ast$ of extreme points $y_1,\dots,y_m \in \rb^d$, and for which $\min_{y \in \mathcal{K}} f'(x_\ast)^\top y$ is attained only by elements of $\mathcal{K}_\ast$. Then:
\BIT
\item For $\rho_k = 1/k$,   $x_k = x_\ast + \frac{1}{k} \Delta_1 + O(1/k^2) $. This implies $f(x_k) - f(x_\ast) = \frac{1}{k} \Delta_1^\top f'(x_\ast) + O(1/k^2)$ and $
f(2x_k - x_{k/2}) - f(x_\ast) = O(1/k^2)$.
\item For $\rho_k = 2/(k+1)$,   $x_k = x_\ast + \frac{1}{k(k+1)} \Delta_2 + O(1/k^2)$. This implies $f(x_k) - f(x_\ast) = O(1/k^2)$ and $
f(2x_k - x_{k/2}) - f(x_\ast) = O(1/k^2)$.
\EIT
The two vectors $\Delta_1$ and $\Delta_2$, are orthogonal (for the dot-product defined by $f''(x_\ast)$) to the span of all $y_i-x_\ast$, $i=1,\dots,m$
\end{proposition}
We now discuss the consequences of the proposition above.

\begin{figure}
\begin{center}
\hspace*{-.2cm}
\includegraphics[scale=.41]{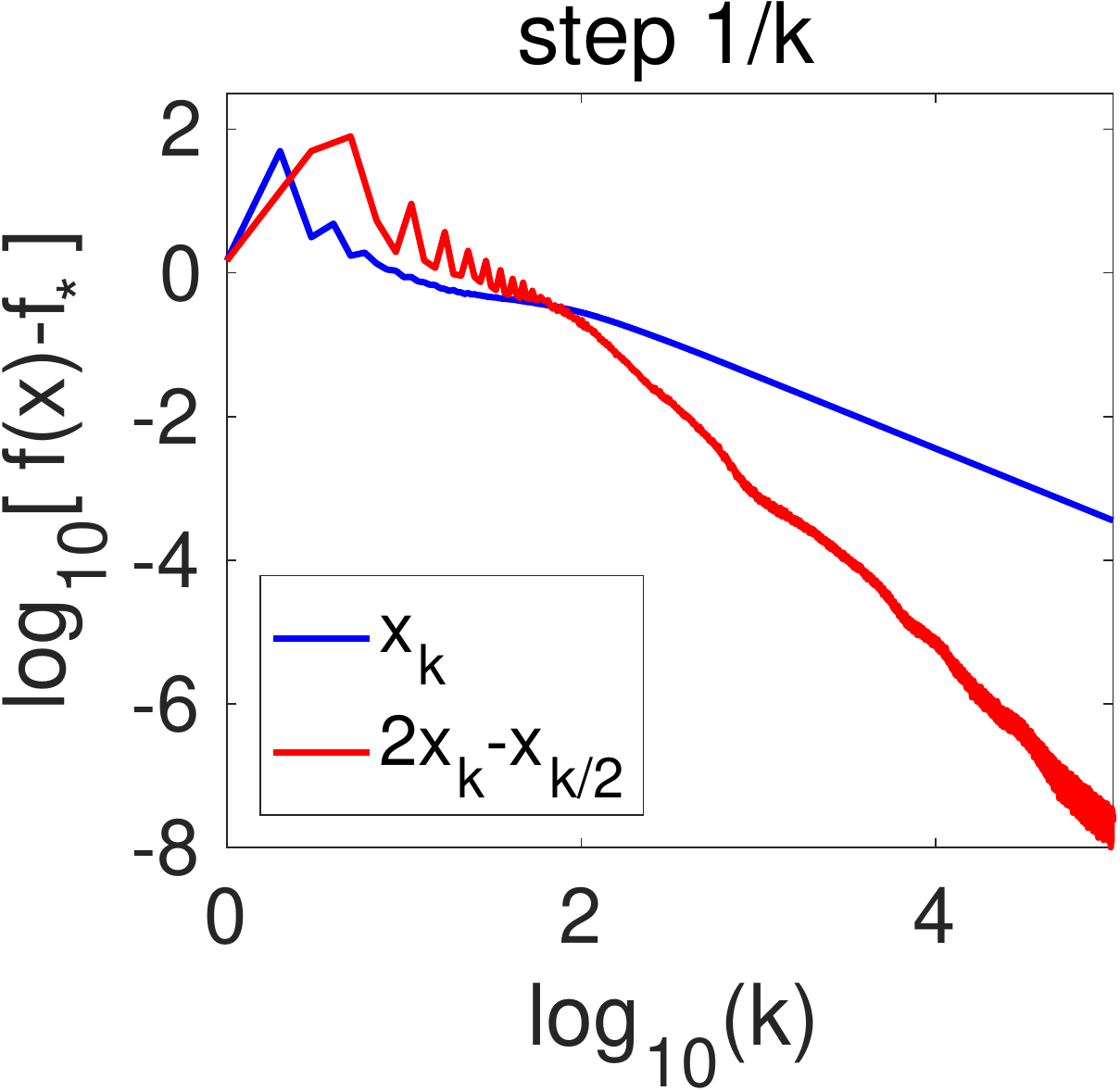}  \hspace*{.4cm}
\includegraphics[scale=.41]{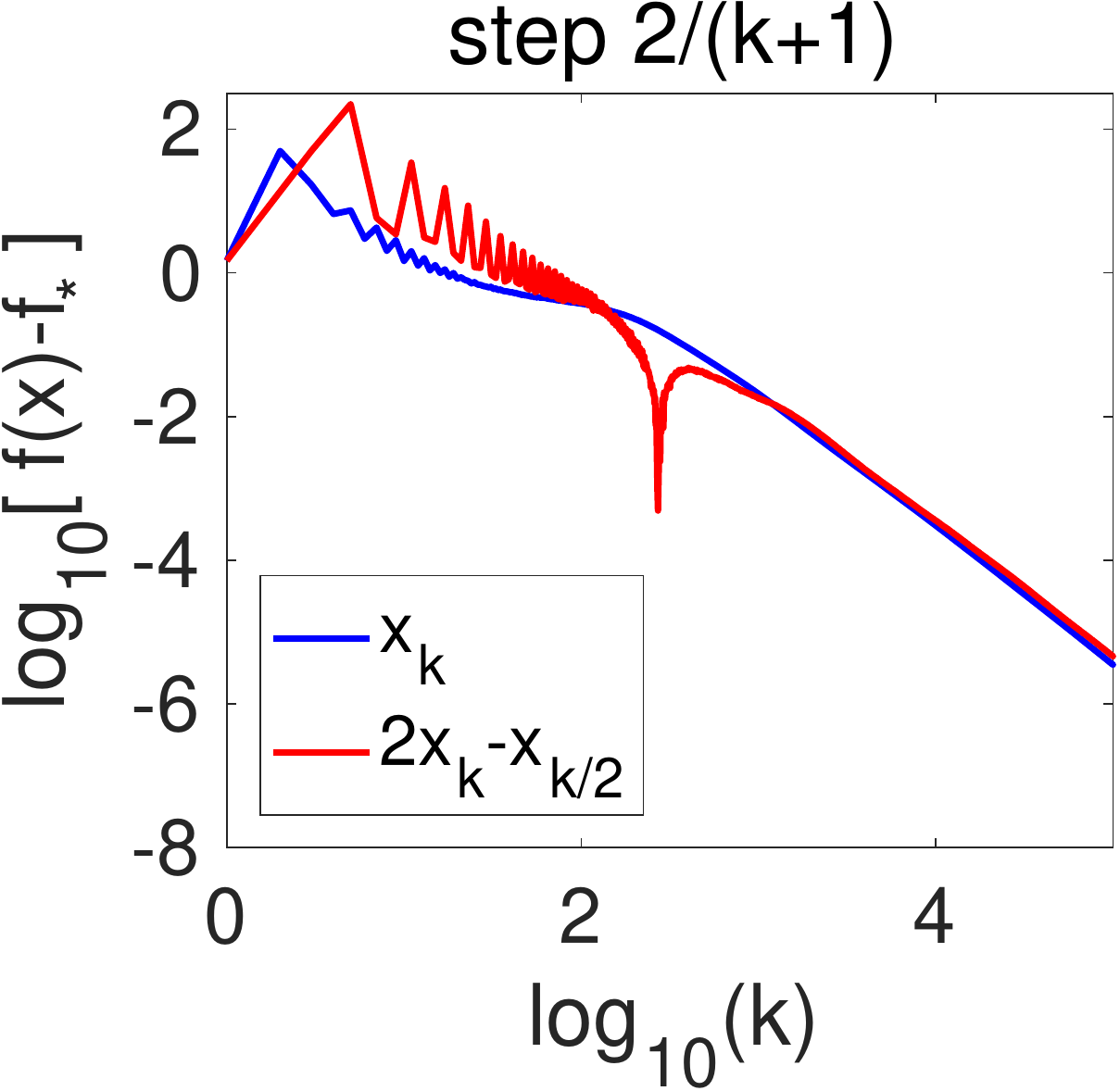}

\vspace*{-.2cm}

\caption{Frank-Wolfe for the dual of robust regression.  We consider the dual of absolute loss regression with $n=400$ observations in dimension $d=200$, sampled from a standard normal distribution, with a squared Euclidean norm penalty. The primal  problem is $\inf_{y \in \rb^d} \frac{1}{n} \| b  - Ax \|_1 + \frac{\lambda}{2} \| y\|_2^2$ while the dual problem is
$\sup_{\| x \|_\infty \leqslant 1}  -f(x)$, with $f(x) = -  \frac{1}{n} x^\top b  + \frac{1}{2 n^2 \lambda} x^\top AA^\top x$. Left: step size $1/k$ with slopes~$-1$ (blue) and $-2$ (red). Right: step size $2/(k+1)$, with slope~$-2$ for the two curves. 
\label{fig:robust}  }
\end{center}
\end{figure}

\paragraph{Step-size $\rho_k = 1/k$.}  Although it leads to a worse performance than the step-size $\rho_k = 2/(k+1)$ both in theory (extra logarithmic term) and in practice, we consider it as (1) this is the ``historical'' step-size~\citep{dunn1978conditional} with an interesting behavior in our set-up, and (2) the corresponding dual algorithm is the subgradient method with plain averaging~\citep[see, e.g.,][]{bach2015duality}, which is sometimes preferred in online learning~\citep{hazan2006logarithmic}.

As shown in Prop.~\ref{prop:FW}, Richardson extrapolation allows to go from an $O(1/k)$ to an $O(1/k^2)$ convergence rate. In the left plots of \myfig{FWlasso} and \myfig{robust}, we can observe the benefits of Richardson extrapolation on two optimization problems with the step-size $\rho_k = 1/k$. Note that: (a) asymptotically, there is provably no extra logarithmic factor like we have for the non-asymptotic convergence rate, and (b) that the Richardson extrapolated iterate may not be within $\mathcal{K}$, but is $O(1/k)$ away from it (in our simulations, we simply make the iterate feasible by rescaling).

\paragraph{Step-size $\rho_k = 2/(k+1)$.}   As shown in Prop.~\ref{prop:FW}, we already get an performance of $O(1/k^2)$ without extrapolation (which is a new asymptotic result for Frank-Wolfe algorithm on polytopes), and Richardson extrapolation does not lead to any acceleration.  In the right plots of \myfig{FWlasso} and \myfig{robust}, we indeed see  no benefits (but no strong degradation).
 
Note that here (and for both step-sizes), higher order Richardson would not lead to further cancellation as within the span of the supporting face, we have an oscillating behavior similar to the left plot of \myfig{oscillating}. Moreover, although we do not have  a proof, the closed loop algorithm exhibits the same behavior as the step $\rho_k = 1/k$, both with and without extrapolation, which is consistent with the analysis of~\citet{canon1968tight}. It would also be interesting to consider the benefits for Richardson extrapolation for strongly-convex sets~\citep{garber2015faster}.

\section{Extrapolation on the Regularization Parameter}
\label{sec:regularization}

In this section, we explore the application of Richardson extrapolation to regularization methods. In a nutshell, regularization allows to make an estimation problem more stable (less subject to variations for statistical problems) or the algorithm faster (for optimization problems). However, regularization adds a bias that needs to be removed. In this section, we apply Richardson extrapolation to the regularization parameter to reduce this bias.
We consider two applications where we can provably show some benefits: (a) smoothing for non-smooth optimization in \mysec{smoothing}, and (b) ridge regression in \mysec{ridge}.

Other applications include the classical use within integration methods, where the technique is often called Richardson-Romberg extrapolation~\citep{gautschi1997numerical}, and for bias removal in constant-step-size stochastic gradient descent for sampling~\citep{durmus2016stochastic} and optimization~\citep{dieuleveut2017bridging}.

\subsection{Smoothing non-smooth problems}

\label{sec:smoothing}

We consider the minimization of a convex function of the form $f(x) = h(x) + g(x)$, where $h$ is smooth and $g$ is non-smooth. These optimization problems are ubiquitous in machine learning and signal processing, where the lack of smoothness can come from (a) non-smooth losses such as max-margin losses used in support vector machines and more generally structured output classification~\citep{taskar2005learning,tsochantaridis2005large}, and (b) sparsity-inducing regularizers~\citep[see, e.g.,][and references therein]{bach2012optimization}. While many algorithms can be used to deal with this non-smoothness, we consider a classical smoothing technique below.

\paragraph{Nesterov smoothing.}
In this section, we consider the smoothing approach of~\citet{nesterov2005smooth} where the non-smooth term is ``smoothed'' into $g_\lambda$, where $\lambda$ is a regularization parameter, and accelerated gradient descent is used to minimize $h+g_\lambda$. 

A typical way of smoothing the function $g$ is to add $\lambda$ times a strongly convex regularizer to the Fenchel conjugate of~$g$ (see an example below); as shown by \citet{nesterov2005smooth}, this leads to a function $g_\lambda$ which has a smoothness constant (defined as the maximum of the largest eigenvalues of all Hessians) proportional to $1/\lambda$, with a uniform error of $\lambda$ between $g$ and $g_\lambda$. Given that accelerated gradient descent leads to an iterate with excess function values proportional to $1/(\lambda k^2)$ after $k$ iterations, with the choice of $\lambda \propto 1/k$, this leads to an excess in function values proportional to $1/k$, which improves on the subgradient method which converges in $O(1/\sqrt{k})$.

\paragraph{Richardson extrapolation.} If we denote by $x_\lambda$ the minimizer of $h+g_\lambda$ and $x_\ast$ the global minimizer of $f=h+g$, if we can show that $x_\lambda = x_\ast + \lambda \Delta + O(\lambda^2)$, then $x^{(1)}_\lambda = 2 x_\lambda - x_{2\lambda} = O(\lambda^2)$ and we can expand $f(x_\lambda^{(1)})  = f(x_\ast)  + O(\lambda^2)$, which is better than the $O(\lambda)$ approximation without extrapolation. 

Then, with $\lambda \propto k^{-2/3}$, to balance the two terms $1/(\lambda k^2)$ and $\lambda^2$,  we get an overall convergence rate for the non-smooth problem of $k^{-4/3}$. We now make this formal for the special (but still quite generic) case of polyhedral functions~$g$, and also consider $m$-step Richardson extrapolation, which leads to a convergence rate arbitrarily close to~$O(1/k^2)$.

\paragraph{Polyhedral functions.} We consider
a polyhedral function of the form
$$g(x) = \max_{i \in \{1,\dots,m\}} a_i^\top x - b_i =   \max_{i \in \{1,\dots,m\}} (A x - b)_i,$$
where $A \in \rb^{m \times d}$ and $b \in \rb^m$. This form includes traditional regularizers such as the $\ell_1$-norm, the $\ell_\infty$-norm, grouped $\ell_1$-$\ell_\infty$-norms~\citep{negahban2008joint}, or more general sparsity-inducing norms~\citep{bach2010structured}.

We consider the smoothing of $g$ as:
$$
g_\lambda(x) = \max_{ \eta \in \Delta_m} \eta^\top ( A x - b) - \lambda \varphi(\eta),
$$
for some strongly convex function $\varphi$, typically, the negative entropy $\sum_{i=1}^m \big\{ \eta_i \log \eta_i - \eta_i \big\}$, or $\frac{1}{2} \| \eta\|_2^2$. 
For our asymptotic expansion, we also need a form a constraint qualification (see proof in Appendix~\ref{app:smoothing}).

\begin{proposition}
\label{prop:nesterov}
Assume $h$ convex, three-times differentiable with bounded third-order derivatives $g$ convex,  and a unique minimizer $x_\ast \in \rb^d$ of $h+g$ such that $h''(x_\ast) $ is positive definite. Assume there exists $\eta_\ast \in \rb^m$ in the simplex $\Delta_m$ (defined as the vectors  in $\rb^m$ with non-negative components summing to one), such that, for the support $I \subset \{1,\dots,m\}$ of $\eta_\ast $ (that is, the set of non-zeros), 
$$
h'(x_\ast) + A^\top \eta_\ast  = 0,
$$
and 
$$\max_{i \in \{1,\dots,m\}} (A x_\ast - b)_i \mbox{ is only attained  for all } i \in I.$$
Assume moreover the submatrix $A_I$ obtained by taking the  the rows of $A$ indexed by $I$ has full rank. 
We denote by $x_\lambda$ a minimizer of $h(x) + g_\lambda(x)$, and $\eta_\lambda$ the corresponding dual variable. Then:
$$
x_\lambda  = x_\ast + \lambda \Delta + O(\lambda^2),
$$
with $\Delta =    h''(x_\ast)^{-1} A_I^\top  \big[ A_I h''(x_\ast)^{-1} A_I^\top  \big]^{-1} (\eta_\ast)_I $ for the quadratic penalty, and a similar expression for the entropic penalty.
\end{proposition}

The proposition above implies that (see detailed proof in Appendix~\ref{app:smoothing} for details) the smoothing technique asymptotically adds a bias of order $\lambda$
$$
f(x_\lambda)  = f(x_\ast)  + O(\lambda),
$$
where we recover the usual upper-bound in $O(\lambda)$, confirming the result from~\citet{nesterov2005smooth}. The key other consequence is that 
$$
f(x_\lambda^{(1)}) = f(2 x_\lambda -x_{2\lambda})  = f(x_\ast)  + O(\lambda^2) ,
$$
which shows the benefits of Richardson extrapolation.

% Assuming $x_\lambda$ and $x_\ast$ are close, we can perform Taylor expansions as:
%$$
%0 = f'(x_\lambda)+g_\lambda'(x_\lambda) = f'(x_\ast) + f''(x_\ast) ( x_\lambda - x_\ast)
%+ g_\lambda'(x_\ast)  + g_\lambda''(x_\lambda) ( x_\lambda - x_\ast) + O( \varepsilon^{-2} \| x_\lambda - x_\ast \|^2 ).
%$$

\paragraph{Multiple-step Richardson extrapolation.} Given that one-step Richardson extrapolation allows to go from a bias of $O(\lambda)$ to $O(\lambda^2)$, a natural extension is to consider $m$-step Richardson extrapolation~\citep{gautschi1997numerical}, that is, a combination of $m+1$ iterates:
$$
x^{(m)}_\lambda = \sum_{i=1}^{m+1} \alpha^{(m)}_i x_{i\lambda},
$$
where the weights $\alpha^{(m)}_i$ are such that the first $m$ powers in the Taylor expansion of $x_\lambda$ cancel out. This can be done by solving the linear system based on the following equations:
\BEA
\label{eq:alpha1}
& \!\! & \sum_{i=1}^{m+1} \alpha^{(m)}_i = 1 \\
\label{eq:alpham} \forall j\in \{1,\dots,m\},  & \!\! & \sum_{i=1}^{m+1} \alpha^{(m)}_i i^{j} = 0.
\EEA
Using the same technique as~\citet[Lemma 3.1]{pages2007multi}, this is a Vandermonde system with a closed form solution (see proof in Appendix~\ref{app:gamma}):
\BEAS
\alpha^{(m)}_i  =  (-1)^{i-1}  { m+1 \choose i} .
\EEAS
We show in Appendix~\ref{app:smoothing} the following proposition:
\begin{proposition}
\label{prop:nesterov2}
On top of assumptions from Prop.~\ref{prop:nesterov}, assume $h$ is $(m+2)$-times differentiable with bounded derivatives. Then
$$f(x_\lambda)  =  f(x_\ast) + O(\lambda^{m+1}).$$
\end{proposition}
Thus, within the smoothing technique, if we consider $\lambda \propto 1/k^{2/(m+2)}$, to balance the terms $1/(\lambda k^2)$ and $\lambda^{m+1}$, we get an error for the non-smooth problem of 
$1/k^{2(m+1)/(m+2)}$, which can get arbitrarily close to $1/k^2$ when $m$ gets large. The downsides are that (a)  the constants in front of the asymptotic equivalent may blow up (a classical problem in high-order expansions), and (b) $m$-step extrapolation requires running the algorithm $m$ times (this can be down in parallel). In our experiments below, 3-step extrapolation already brings in most of the benefits.

\begin{figure}
\begin{center}
\hspace*{-.2cm}
\includegraphics[scale=.41]{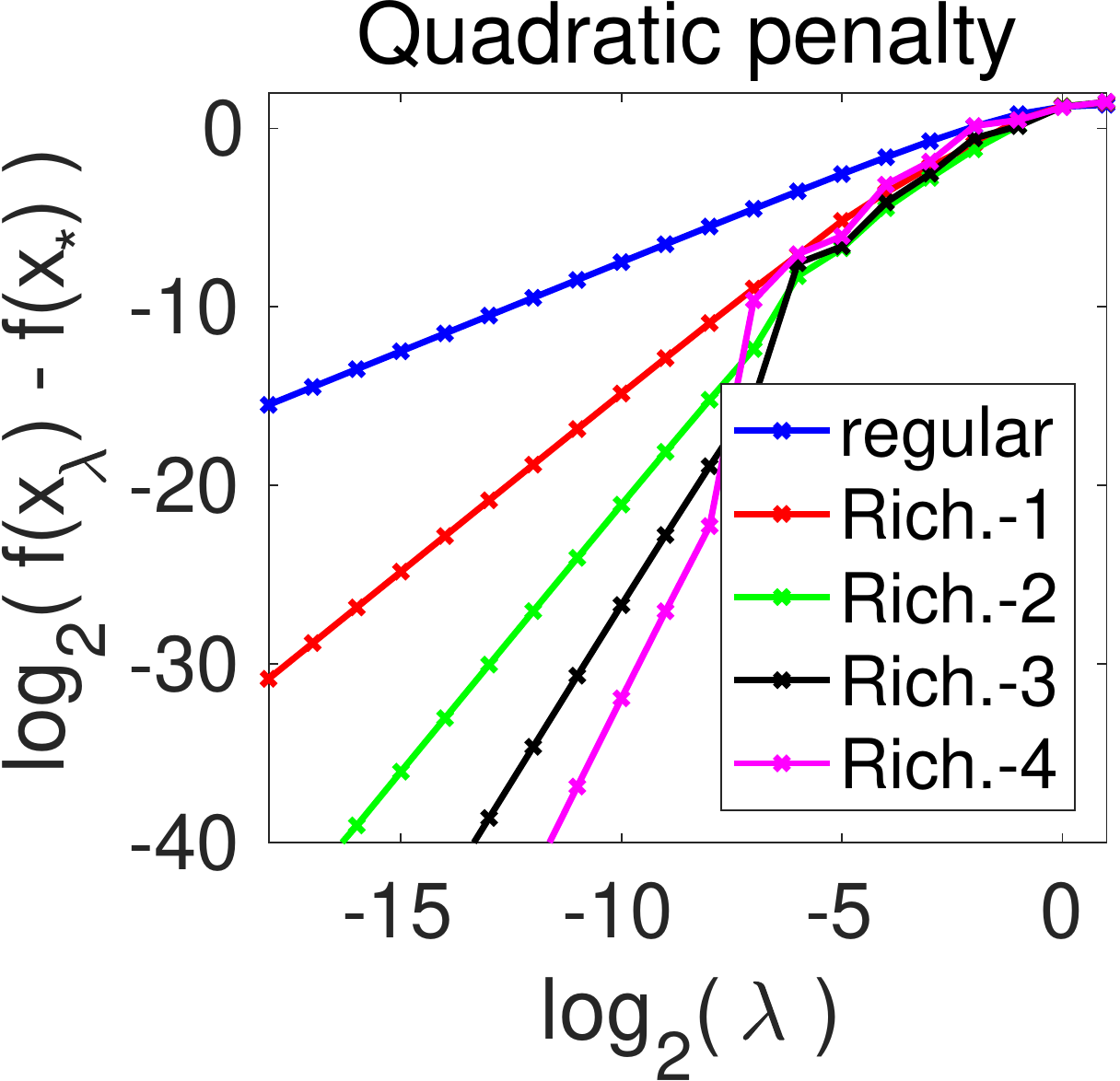} \hspace*{.4cm}
\includegraphics[scale=.41]{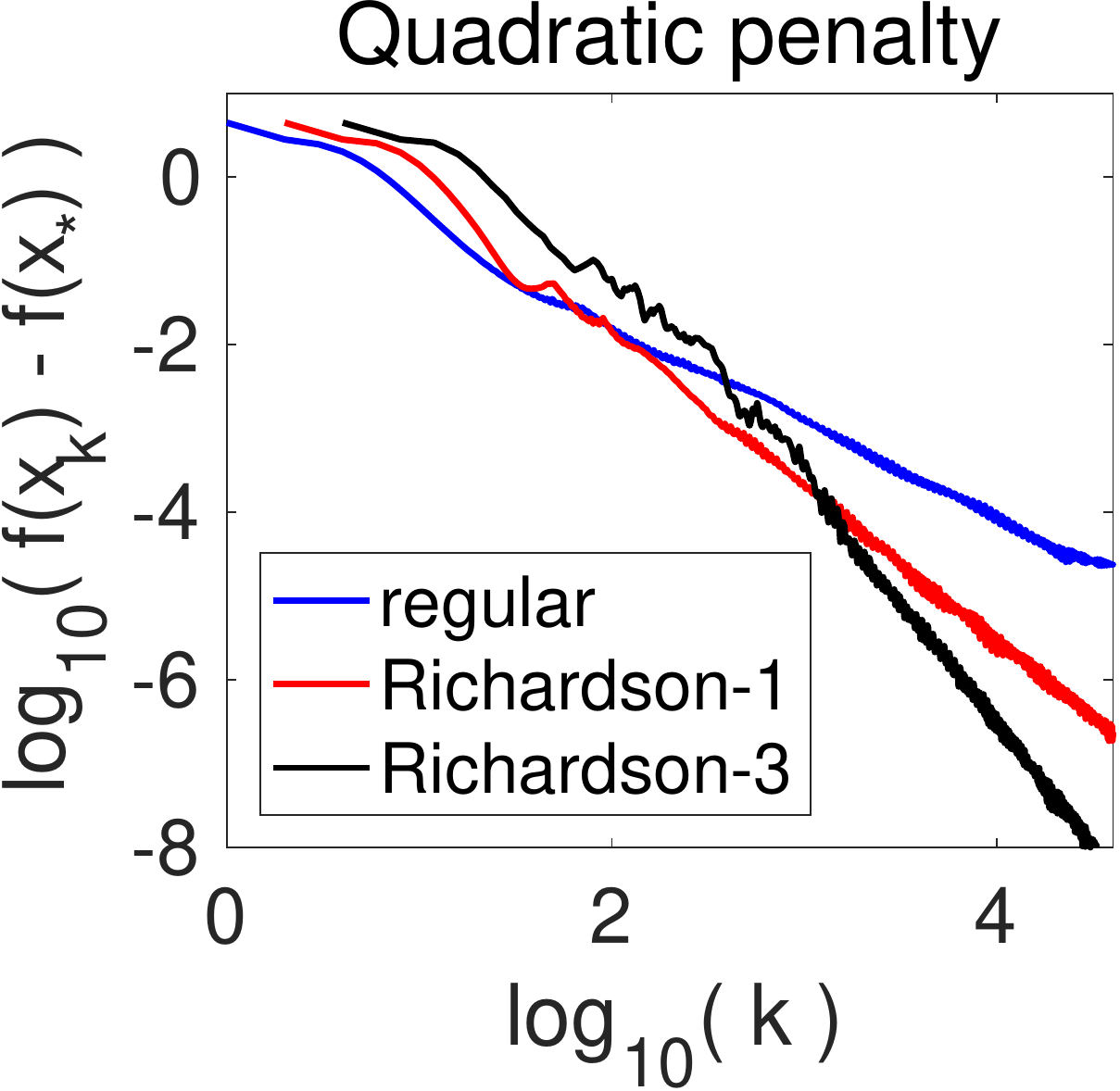}

\vspace*{-.2cm}

\caption{Richardson extrapolation for Nesterov smoothing on a penalized Lasso problem, with regularization by the quadratic penalty.  Left:  dependence of $f(x_\lambda) - f(x_\ast)$ on $\lambda$, for Richardson extrapolation of order $m$, we indeed recover a slope of $m+1$ in the log-log plot. Right: Optimization error vs.~number of iterations; where we go from a slope of -1 (blue curves) to improved slopes of $-4/3$ (red curve) and $-8/5$ (black curve). See text for details.\label{fig:smoothing_quad}}
\end{center}
\end{figure}

\paragraph{Experiments.}
We consider the penalized Lasso problem:  
$$
\min_{x \in \rb^d} \frac{1}{2n} \sum_{i=1}^n ( b_i - x^\top a_i)^2+  \lambda \| x\|_1,
$$
where $(a_i,b_i) \in \rb^d \times \rb$ for $i=1,\dots,n$, for $d=100$ and $n=100$, and with input data distributed as a standard normal vector. We use either a dual entropic penalty or a dual quadratic penalty for smoothing each component $|x_j|$ of the $\ell_1$-norm of $x$. Plots for the quadratic penalty are presented here in \myfig{smoothing_quad}, while plots for the entropic penalty are presented in \myfig{smoothing_ent} in Appendix~\ref{app:smoothing} with the same conclusions.

In the left plot of \myfig{smoothing_quad}, we illustrate the dependence of $ f(x_\lambda) - f(x_\ast) $ on $\lambda$ for Richardson extrapolation with various orders, while in the right plot of \myfig{smoothing_quad}, we study the effect of extrapolation to solve the non-smooth problem. For a series of regularization parameters equal to $2^{i}$ for $i$ between $-18$ and $1$ (sampled every $1/5$), we run accelerated gradient descent on $h+g_\lambda$ and we plot the value of $f(x)-f(x_\ast)$ for the various estimates, where for each number of iterations, we minimize over the regularization parameter. This is an oracle version of  varying $\lambda$ as a function of the number of iterations (a detailed evaluation where $\lambda$ depends on $k$ could also be carried out). In \myfig{smoothing_quad}, we plot the excess function values as a function of the number of iterations, \emph{taking into account that $m$-step Richardson extrapolation requires $m$-times more iterations}. We see that we get a strong improvement approaching $1/k^2$.

\paragraph{From non-linear programming to linear programming.}
When we use the entropic penalty, the smoothing framework is generally applicable in most non-linear programming problems \citep[see, e.g.,][]{cominetti1994stable}. It is interesting to note that typically when applying the entropic penalty, the deviation to the global optimizer is going to zero \emph{exponentially} in $-1/\lambda$ for some of the components (see a proof for our particular case in Appendix~\ref{app:smoothing}), but not for the corresponding dual problem (which is our primal problem).

Another classical instance of entropic regularization in machine learning leads to the Sinkhorn algorithm for computing optimal transport plans~\citep{cuturi2013sinkhorn}. For that problem, the entropic penalty is put directly on the original problem, and the deviation in estimating the optimal transport plan can be shown to be asymptotically exponential in $-1/\lambda$~\citep{cominetti1994asymptotic}, and thus there Richardson extrapolation is not helpful (unless one wants to estimate the Kantorovich dual potentials).

\subsection{Improving bias in ridge regression}
\label{sec:ridge}
We consider the ridge regression problem, that is, $w_\lambda$ is the unique minimizer of
$$ \min_{ w \in \rb^d} \ \frac{1}{2n} \| y - \Phi w \|_2^2 + \frac{\lambda}{2} \| w\|_2^2,$$
where $\Phi \in \rb^{n\times d}$ is a feature vector and $y \in \rb^n$   a vector of responses~\citep{friedman2001elements}. The solution may be obtained in closed form by solving the normal equations, as $w_\lambda =  ( \Phi^\top \Phi + n\lambda \idm)^{-1} \Phi^\top y$.

The regularization term $\frac{\lambda}{2} \| w\|_2^2$ is added to avoid overfitting and control the variability of $w_\lambda$ due to the randomness in the training data (the higher the $\lambda$, the more control); however, it does create a bias that goes down as $\lambda$ goes to zero. Richardson extrapolation can be used to reduce this bias. We thus consider $w^{(1)}_\lambda = 2 w_\lambda - w_{2\lambda}$ and more generally
$$
w^{(m)}_\lambda = \sum_{i=0}^m \alpha^{(m)}_i w_{i \lambda},
$$
with the same weights as defined in \eq{alpha1} and \eq{alpham}. In order to compute $w^{(m)}_\lambda$, either $m$ ridge regression problems can be solved, or a closed-form spectral formula can be used based on a single singular value decomposition of the kernel matrix (see Appendix~\ref{app:gamma} for details).

\paragraph{Theoretical analysis.}

Following~\citet{bach2013sharp}, we assume for simplicity that $\Phi$ is deterministic and that $y = z + \varepsilon$, where $\varepsilon$ has zero mean and covariance matrix $\sigma^2 \idm$. We
consider the in-sample error of $\hat{y}_\lambda=\Phi w_\lambda = K ( K + n\lambda \idm)^{-1} y = \hat{H}_\lambda y$, where $K = \Phi \Phi^\top$ is the usual kernel matrix, and $\hat{H}_\lambda$ is the smoothing matrix, which is equal to $\idm$ for very small $\lambda$ and equal to zero for very large $\lambda$. We consider the so-called ``in-sample generalization error'', that is, we want to minimize
\BEAS
\frac{1}{n} \E \| \hat{y}_\lambda - z\|_2^2&  = &   {\rm bias}(\hat{H}_\lambda) + {\rm variance} (\hat{H}_\lambda),
\EEAS
where 
\BEAS
  {\rm bias}(\hat{H}_\lambda)  & =& \frac{1}{n} \| (\hat{H}_\lambda - \idm)z \|_2^2 \\
    {\rm variance}(\hat{H}_\lambda)  & =  &  \frac{\sigma^2}{n} \tr \hat{H}_\lambda^2.
\EEAS
The bias term is increasing in $\lambda$, while the variance term in decreasing in $\lambda$, and there is thus a trade-off between these two terms. To find the optimal $\lambda$, assumptions need to be made on the problem, regarding the eigenvalues of $K$, and the components of $z$ in the eigenbasis of $K$. That is, following the notations of~\citet[Section 4.3]{bach2013sharp}, we assume that the eigenvalues of $K$ are $\Theta( n \mu_i)$ (that is bounded from above and below by constants times $n\mu_i$) and the coordinates of $z$ in the eigenbasis of $K$ are $\Theta( \sqrt{n  \nu_i})$. The precise trade-off depends on the rates at which $\mu_i$ and $\nu_i$ decay to zero.

A classical situation is $\mu_i \sim i^{-2\beta}$ and $\nu_i \sim i^{-2\delta}$, where $\beta>1/2$ and $\delta>1/2$ (to ensure finite energy). As detailed in Appendix~\ref{app:ridge}:
\BIT
\item  the variance term is equivalent to $\frac{\sigma^2}{n} \lambda^{-1/2\beta}$ and does not depend on $z$ or $\delta$;
\item  the bias term depends  on both $\delta$ and $\beta$: for  signals which are not too smooth (i.e., not too fast decay of $\nu_i$, and thus small $\delta$), that is if $\delta < 2\beta + 1/2$, then the bias term is equivalent to $\lambda^{(2\delta-1)/2\beta}$ and we can thus find the optimal $\lambda$ as $(\sigma^2/n)^{\beta/\delta}$ leading to predictive performance of $(\sigma^2/n)^{1 - 1/2\delta}$, which happens to be optimal~\citep{caponnetto2007optimal}. However, when $\delta > 2\beta + 1/2$, a phenomenon called ``saturation'' occurs,  and the bias term is equivalent to $\lambda^2$ (independent of $\delta$), and the optimized predictive performance is
$(\sigma^2/n)^{1 - 1/(4\beta + 1)}$, which is not optimal anymore.
\EIT
As shown in  Appendix~\ref{app:ridge}, by reducing the bias, with $m$-step Richardson interpolation, we can show that the variance term is bounded by a constant times the usual one, while the bias term is equivalent to $\lambda^{(2\delta-1)/2\beta}$ for a wider range of $\delta$, that is, $ \delta <  2(m+1)\beta + 1/2$, which recovers for $m=0$ the non-extrapolated estimate. This leads to optimal statistical performance for a wider range of problems.

\begin{figure}
\begin{center}
\includegraphics[scale=.4]{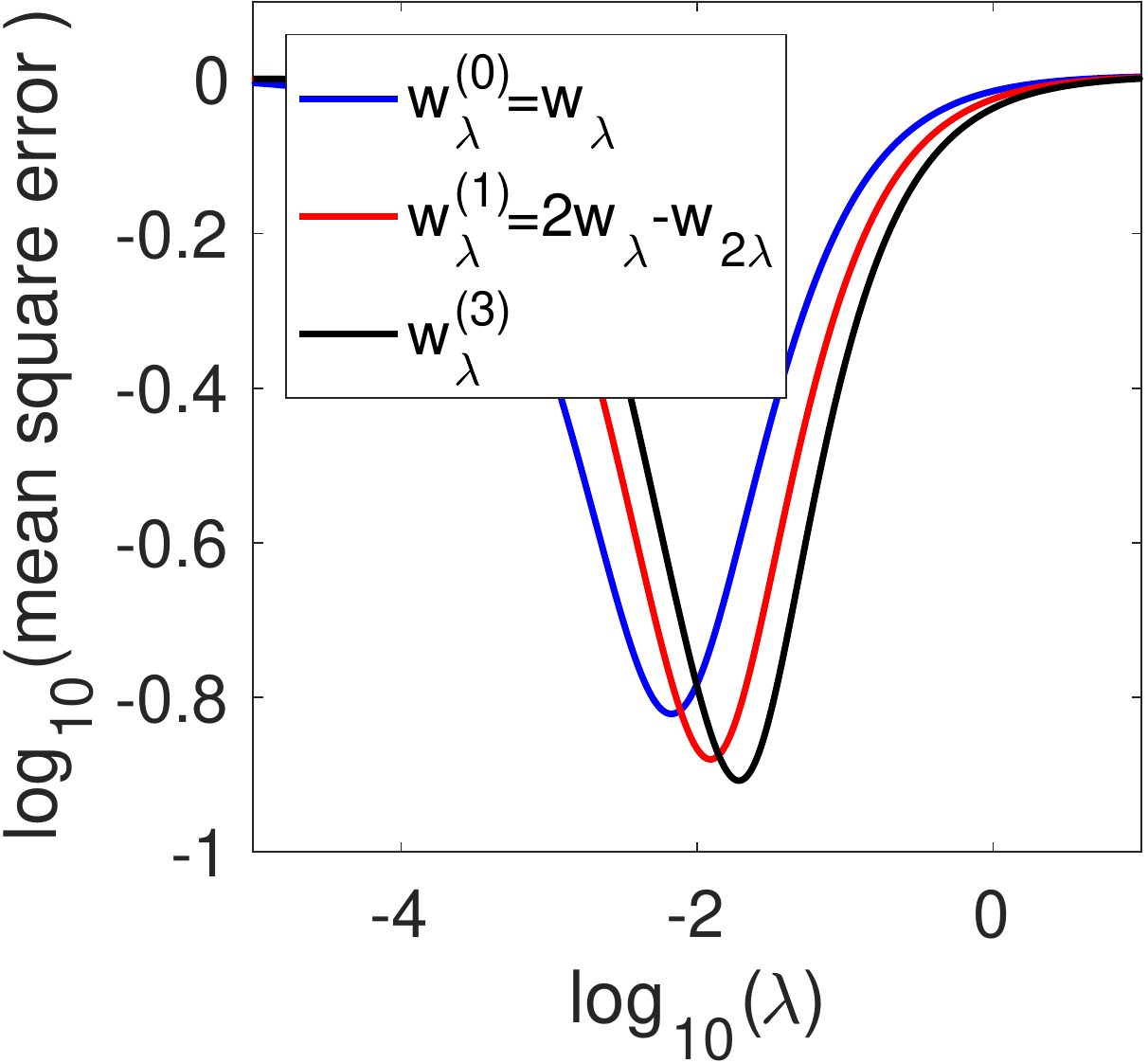} \hspace*{.1cm}
\includegraphics[scale=.4]{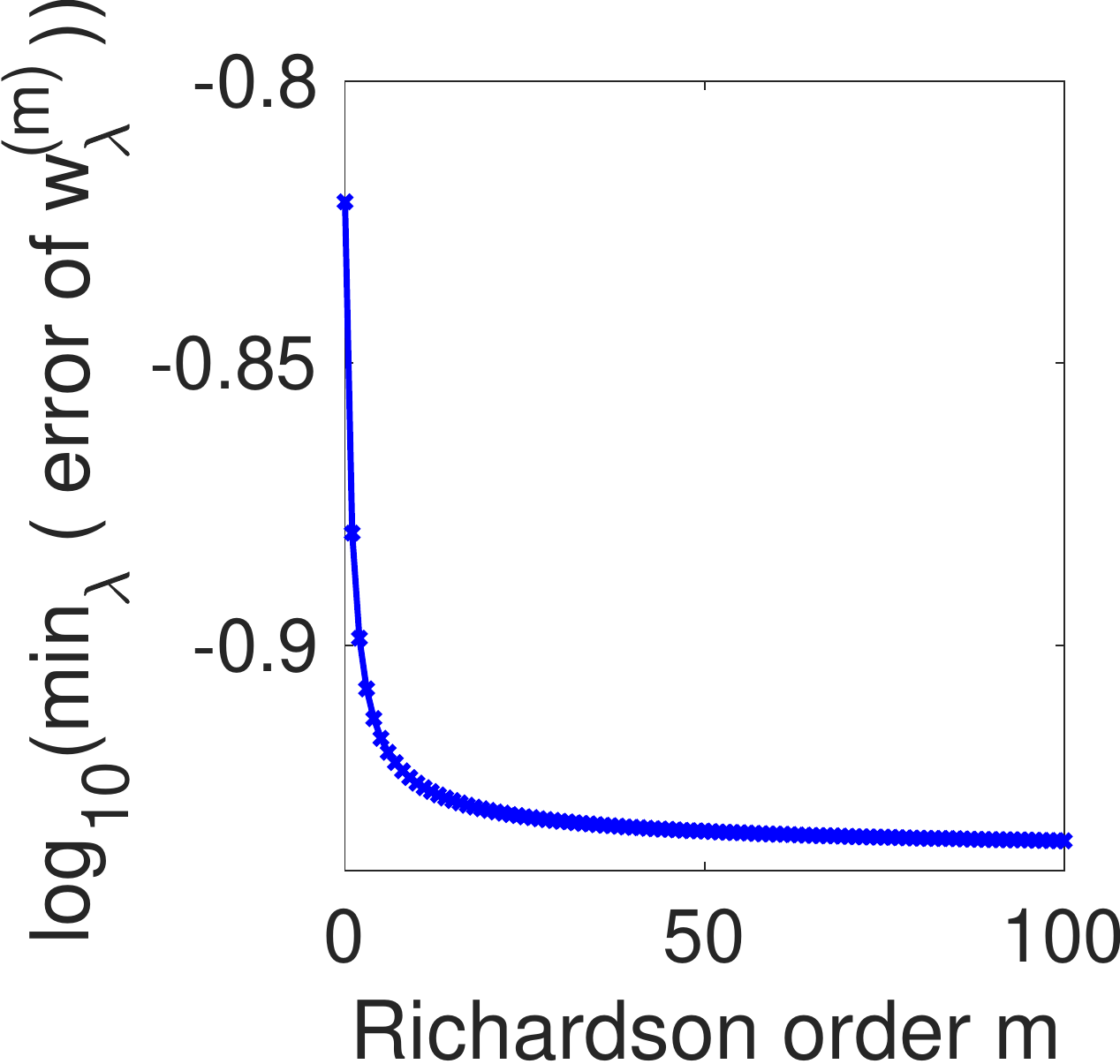}

\vspace*{-.2cm}

\caption{Left: regularization path for the classical iterate $w_\lambda$, one step of Richardson $w_\lambda^{(1)}$ and 3 steps $w_\lambda^{(3)}$. Right: optimal error as a function of the order of the Richardson step. \label{fig:ridge}}
\end{center}
\end{figure}

\paragraph{Experiments.}
As an illustration, we consider a ridge regression problem with data uniformly sampled on the unit sphere in dimension $d=40$ with $n=200$ observations, and $y$ generated as a linear function of the input plus some noise. We consider the rotation invariant kernel equal to the expectation $k(x,x') = \E_{\tau} (1+x^\top x') \sigma(\tau^\top x) \sigma(\tau^\top x')$, for $\tau$ uniform on the sphere. This is equal to, for $\sigma(\alpha) = 1_{\alpha>0}$~\citep[see][]{cho2009kernel,bach2017breaking}:
$$
k(x,y) \propto (1+x^\top x') \big[ \pi - \arccos (x^\top x') \big].
$$
When the number of observations $n$ tends to infinity, the eigenvalues of $\frac{1}{n} K$ are known to  converge to the eigenvalues of a certain infinite-dimensional operator~\citep{koltchinskii2000random}. As shown by~\citet{bach2017breaking}, the corresponding eigenvalues of the kernel matrix decay as $i^{-1-1/d}$. We consider $z$ generated as  a linear function so that $\nu_i$ has a finite number of non zero components in the eigenbasis of $K$. 

In the left plot of \myfig{ridge}, we consider $1$-step and $3$-step Richardson extrapolation and plot the generalization error (averaged over 10 replications) as a function of the regularization parameter: we can see that as expected, (a) with extrapolation the curves move to the right (we can use a larger $\lambda$ for a similar performance, which is advantageous as iterative algorithms are typically faster), and (b) the minimal error is smaller (which is true here because we learn a smooth function). In the right plot of \myfig{ridge}, we study the effect of increasing the order $m$ of extrapolation, showing that the larger the better with some saturation. With $m$ infinite, there will be overfitting as the corresponding spectral filter is non-stable, but this happens very slowly (see Appendix~\ref{app:gamma} for details).

\section{Conclusion}

In this paper, we presented various applications of Richardson extrapolation to machine learning optimization and estimation problems, each time with an asymptotic analysis showing the potential benefits. For example, when using the number of iterations of an iterative algorithm to perform extrapolation, we can accelerate Frank-Wolfe algorithms to have an asymptotic rates of $O(1/k^2)$ on polytopes for the step-size $\rho_k = 1/k$ and locally strongly-convex functions (this is achieved without extrapolation for the step-size $\rho_k  = 2/(k+1)$). When extrapolating based on the regularization parameter, we can accelerate Nesterov smoothing technique to have asymptotic rates close to $O(1/k^2)$.

We also highlighted situations where Richardson extrapolation does not bring any benefits (but does not degrade performance much), namely when applied to accelerated gradient descent or the Sinkhorn algorithm for optimal transport.

The analysis in this paper can be extended in a number of ways: (1) while the paper has focused on asymptotic analysis for simplicity, non-asymptotic analysis could be carried out to study more finely when acceleration starts, (2)~we have focused on deterministic optimization algorithms, and extensions to stochastic algorithms could be derived, along the lines of the work of~\citet{dieuleveut2017bridging}, (3) we have primarily focused on convex optimization algorithms but non-convex extensions, like done by~\citet{scieur2018nonlinear} for Anderson acceleration, could also lead to acceleration

\subsection*{Acknowledgements}

This work was funded in part by the French government under management of Agence Nationale de la Recherche as part of the ``Investissements d'avenir'' program, reference ANR-19-P3IA-0001 (PRAIRIE 3IA Institute). We also acknowledge 
 support the European Research Council (grant SEQUOIA 724063).

\bibliography{richardson}

\newpage
\appendix

\section{Preliminary considerations}
\label{app:preliminary}

We first start with lemmas that we will need in subsequent proofs; the second one shows strong convexity on a level set once we assume that the Hessian at optimum is positive definite.

\begin{lemma}
\label{lemma:1}
Assume   $f$  convex, three-times differentiable with bounded third-order derivatives, and a point $x_\ast $  such that $f''(x_\ast) $ is positive definite. Then there exists $c>0$ such that for any $x \in \rb^d$.
$$\frac{1}{2} (x - x_\ast)^\top f''(x_\ast) ( x - x_\ast) \leqslant c \Rightarrow f''(x) \succcurlyeq \frac{1}{2} f''(x_\ast).$$
\end{lemma}
\begin{proof}
Since $f''(x_\ast) $ is positive definite, $\lambda_{\min}(f''(x_\ast)) >0$, and $\frac{1}{2} (x - x_\ast)^\top f''(x_\ast) ( x - x_\ast) \leqslant c$ implies that
$ \| x - x_\ast\|_2^2 \leqslant \frac{2c}{\lambda_{\min}(f''(x_\ast))}$. Thus, since third order derivatives of $f$ are bounded,  if $\frac{1}{2} (x - x_\ast)^\top f''(x_\ast) ( x - x_\ast) \leqslant c$, we have $\| f''(x) - f''(x_\ast) \|_{\rm op} \leqslant A c$ for some constant $A$, and thus
$$
f''(x) \succcurlyeq f''(x_\ast) - Ac \idm \succcurlyeq f''(x_\ast) - \frac{Ac}{\lambda_{\min}(f''(x_\ast))}  f''(x_\ast) \succcurlyeq\frac{1}{2} f''(x_\ast)
$$
if $Ac \leqslant \frac{1}{2} \lambda_{\min}(f''(x_\ast))$.
\end{proof}

\begin{lemma}
\label{lemma:2}
Assume   $f$  convex, three-times differentiable with bounded third-order derivatives, and that $x_\ast \in \rb^d$ is a minimizer of $f$ on $\rb^d$, such that $f''(x_\ast) $ is positive definite. Then $x_\ast$ is the unique minimizer of $f$  and there exists $c>0$ such that for any $x \in \rb^d$,
$$f(x) - f(x_\ast)  \leqslant c \Rightarrow f''(x) \succcurlyeq \frac{1}{2} f''(x_\ast).$$
\end{lemma}
\begin{proof}
Using the above Lemma~\ref{lemma:1}, then there exists $c >0$ such that 
$$\frac{1}{2} (x - x_\ast)^\top f''(x_\ast) ( x - x_\ast) \leqslant c \Rightarrow f''(x) \succcurlyeq \frac{1}{2} f''(x_\ast).$$
We now show that  $f(x) - f(x_\ast) \leqslant c/2$ implies  that we stay in the region where  $\frac{1}{2} (x - x_\ast)^\top f''(x_\ast) ( x - x_\ast) \leqslant c $. Indeed, using the Taylor formula with integral remainder, for $\frac{1}{2} (x - x_\ast)^\top f''(x_\ast) ( x - x_\ast)  \leqslant c $ (where the lower bound on Hessians above holds), we have
 \BEAS
 f(x) - f(x_\ast)  & = &  0 +   \int_0^1(x-x_\ast)^\top f''(x_\ast + t(x-x_\ast))(x-x_\ast) (1-t)dt  \\
 & \geqslant 
 &  \frac{1}{2} \Big( \int_0^1(1-t)dt  \Big) (x-x_\ast)^\top f''(x_\ast)(x-x_\ast)  = \frac{1}{4} (x-x_\ast)^\top f''(x_\ast)(x-x_\ast)  
 \EEAS 
  and thus
$f(x) - f(x_\ast) \leqslant c/2$ implies $\frac{1}{2} (x - x_\ast)^\top f''(x_\ast) ( x - x_\ast) \leqslant c $, and we get the desired result.
\end{proof}

\section{Proof of Prop.~\ref{prop:gd} (averaged gradient descent)}
\label{app:gd}

In this particular case of unconstrained gradient descent, $f(x_k) - f(x_\ast) \leqslant \frac{1}{\gamma k} \| x_0 - x_\ast\|^2$, as soon as $\gamma \leqslant 1/L$~\citep{nesterov2013introductory}. This implies from Lemma~\ref{lemma:2} in Appendix~\ref{app:preliminary}, that for $k$ larger than some $k_0$, all iterates are such that $f''(x_k) \succcurlyeq \frac{1}{2} f''(x_\ast)$, and thus, after that $k$, we are in the strongly convex case, where $\| x_k - x_\ast \| \leqslant c' \rho^k $, where $c', \rho$ depend on the lowest eigenvalue $\mu$ of $f''(x_\ast)$ as $c' = \| x_{k_0} - x_\ast\|\rho^{-k_0}$ and $\rho = ( 1 - \gamma \mu / 2)$.

Thus, with $\bar{x}_k = \frac{1}{k} \sum_{i=0}^{k-1} x_i$,
 $k (\bar{x}_k - x_\ast)$ tends to $ \sum_{i=0}^\infty ( x_i - x_\ast)$ when $k \to +\infty$ (since the series is convergent), and
$$
k (\bar{x}_k - x_\ast) -  \sum_{i=0}^\infty ( x_i - x_\ast) = -\sum_{i=k}^\infty (x_i - x_\ast),
$$
leading to $k (\bar{x}_k - x_\ast) -  \sum_{i=0}^\infty ( x_i - x_\ast) = O ( \rho^k )$, and thus, with $\Delta = \sum_{i=0}^\infty ( x_i - x_\ast)$ (which is hard to compute a priori),
$$
\bar{x}_k  =  x_\ast + \frac{1}{k} \Delta + O(\rho^k).
$$

\section{Proof of Prop.~\ref{prop:FW} (Frank-Wolfe)}
\label{app:FW}

\paragraph{Preliminary remarks.}
We consider the step-sizes $\rho_k = \frac{1}{k}$ and $\rho_k = \frac{2}{k+1}$, for which the respective convergence rates for $f(x_k) - f(x_\ast)$ are of the form $\frac{c}{k}$ and $c \frac{ \log k}{k}$, for constants $c$ depending on the smoothness of $f$ and the diameter of the compact set~\citep[see, e.g.,][]{jaggi2013revisiting}. When running Frank-Wolfe (with any of the classical versions with open-loop step-sizes), we thus have $f(x_k) - f(x_\ast) = O((\log k )^\beta/k)$ with $\beta \in \{0,1\}$. 

Because of affine invariance of the Frank-Wolfe algorithm (and because $f''(x_\ast)$ is invertible), we can assume without loss of generality that $f''(x_\ast) = \idm$. Moreover, the constraint qualification implies that $f'(x_\ast) \neq 0$, thus, using Taylor expansion with integral remainder like in Lemma~\ref{lemma:2}, if $\| x - x_\ast\|^2$ is small enough then
$$
f(x) - f(x_\ast) \geqslant f'(x_\ast)^\top ( x - x_\ast) + \frac{1}{4} \| x - x_\ast\|_2^2.
$$
Since $x_\ast$ is the minimizer of $f$ on $\mathcal{K}$ and $x \in \mathcal{K}$, then $f'(x_\ast)^\top ( x - x_\ast) \geqslant 0$ and thus, we have that 
 $x_k - x_\ast = O((\log k)^{\beta/2}/\sqrt{k})$ and (from Lemma~\ref{lemma:1}), that $f$ is locally strongly-convex.

\paragraph{Analysis of Frank-Wolfe.}
The assumption which is made  implies that there exists $\sigma >0$ such that the ball of center $x_\ast$ and radius $\sigma$ intersected with the affine hull of $y_1,\dots,y_m$ is included in the convex hull of $y_1,\dots,y_m $, as well $\alpha \in (0,1)$ such that if $\cos (f'(x_\ast),z) \geqslant 1-\alpha$, then $\min_{y \in \mathcal{K}} z^\top y$ is attained only by elements of the convex hull of $y_1,\dots,y_m$. 

Thus, for $k$ large enough, that is greater than some~$k_0$ (which can be quantified from $\sigma$, $\alpha$ and other quantities), all elements of $\arg \min_{y \in \mathcal{K}} f'(x_{k-1})^\top y$ are in the convex hull of $y_1,\dots,y_m$, that is, only the correct extreme points are selected. Denoting $\Pi$ the orthogonal projection on the span of $y_1-x_\ast,\dots,y_m-x_\ast$, we have:
$$
x_{k} = ( 1- \rho_{k}) x_{k-1} + \rho_{k} \bar{x}_{k} \mbox{ where } \bar{x}_{k} \in \arg\min_{y \in \mathcal{K}} f'(x_{k-1})^\top y,
$$
and thus, subtracting $x_\ast$:
$$
x_{k} -x_\ast = ( 1- \rho_{k}) (x_{k-1}-x_\ast) + \rho_{k} (\bar{x}_{k} - x_\ast),$$
leading to, using the projections $\Pi$ and $\idm - \Pi$:
$$
\Pi ( x_{k} -x_\ast)= ( 1- \rho_{k}) \Pi (x_{k-1} -x_\ast) + \rho_{k}  (\bar{x}_{k} - x_\ast), $$
and, because $\bar{x}_{k} - x_\ast$ is in the span of $y_1-x_\ast,\dots,y_m-x_\ast$,
$$
(\idm - \Pi ) (  x_{k} - x_\ast) = ( 1- \rho_{k}) (\idm - \Pi )  ( x_{k-1} - x_\ast) .$$
We now consider these two terms separately.

\paragraph{Convergence of $(\idm - \Pi ) (  x_{k} - x_\ast) $.}
For $\rho_{k} = \frac{2}{k+1}$, we have, in closed form for $k \geqslant k_0$:
$$
(\idm - \Pi ) (  x_{k} - x_\ast)  = \frac{k-1}{k+1}(\idm - \Pi ) (  x_{k-1} - x_\ast)  = \frac{k_0 (k_0 + 1)}{k(k+1)}  (\idm - \Pi ) (  x_{k_0} - x_\ast) .
$$
For $\rho_{k} = \frac{1}{k}$, we have:
$$
(\idm - \Pi ) (  x_{k} - x_\ast)    = \frac{k-1}{k}(\idm - \Pi ) (  x_{k-1} - x_\ast)   = \frac{k_0}{k} (\idm - \Pi ) (  x_{k_0} - x_\ast) .
$$

\paragraph{Convergence of $\Pi (  x_{k} - x_\ast) $.}

We now look at the convergence of $\Pi ( x_{k}- x_\ast)$. 
We have
\BEAS
\| \Pi (x_{k} - x_\ast) \|^2 & = & \|  ( 1- \rho_{k}) \Pi ( x_{k-1} - x_\ast) + \rho_{k} ( \bar{x}_{k}  - x_\ast) \|^2 \\
& = &  ( 1- \rho_{k})^2 \|  \Pi (x_{k-1}  -x_\ast)\|^2 +  \rho_k^2 \|   \bar{x}_{k} - x_\ast \|^2
+ 2 ( 1-\rho_k) \rho_k (\bar{x}_k - x_\ast)^\top   \Pi ( x_{k-1}  - x_\ast).
\EEAS
Because of the ball assumption (that is the existence of a ball around $x_\ast$ that is contained in the supporting face of $\mathcal{K}$), we have
$$
 f'(x_{k-1})^\top ( x_{k-1} - \bar{x}_k)  = \max_{ y \in \mathcal{K}}  f'(x_{k-1})^\top ( x_{k-1} - y)  
 \geqslant   f'(x_{k-1})^\top ( x_{k-1} - x_\ast) + \| \Pi f'(x_{k-1})\| \sigma ,
 $$
 which leads to
 $$
  f'(x_{k-1})^\top ( x_\ast - \bar{x}_k) \geqslant   \| \Pi f'(x_{k-1})\| \sigma.
  $$
 
 Moreover, using a Taylor expansion with $f''(x_\ast)=\idm$, we have
   $f'(x_{k-1}) = f'(x_\ast) + f''(x_\ast) ( x_{k-1} - x_\ast) + O((\log k)^\beta/k) =  f'(x_\ast) + ( x_{k-1} - x_\ast) + O((\log k)^\beta/k)$. Moreover, by optimality of $x_\ast$, $f'(x_\ast)^\top( \bar{x}_k - x_\ast) = 0$. Therefore
 we have, for $\rho_k = O(1/k)$:
\BEAS
\| \Pi x_{k} - x_\ast \|^2 
 & \leqslant & ( 1-\rho_k) ^2 \|  \Pi x_{k-1}  -x_\ast\|^2
 +  \rho_k^2  {\rm diam}(\mathcal{K})^2
 -  2 ( 1-\rho_k) \rho_k  \|\Pi  f'(x_{k-1})\| \sigma 
+ O((\log k)^\beta/k^2)
\\
& \leqslant & ( 1-\rho_k) ^2 \|  \Pi x_{k-1}  -x_\ast\|^2
 -  2 ( 1-\rho_k) \rho_k  \|  \Pi x_{k-1}  -x_\ast\| \sigma 
+ O((\log k)^\beta/k^2).
\EEAS
For $\rho_k = 1/k$, we get
\BEAS
k^2 \| \Pi x_{k} - x_\ast \|^2 
& \leqslant  & ( k-1) ^2 \|  \Pi x_{k-1}  -x_\ast\|^2
 -  2 \sqrt{ (k-1) \|  \Pi x_{k-1}  -x_\ast\|^2}  \sigma 
+ O( \log k).
\EEAS
For $\rho_k = 2/(k+1)$, we get 
\BEAS
k^2 \| \Pi x_{k} - x_\ast \|^2 
& \leqslant  & ( k-1) ^2 \|  \Pi x_{k-1}  -x_\ast\|^2
 -  2 \sqrt{ (k-1) \|  \Pi x_{k-1}  -x_\ast\|^2}  \sigma 
+ O(1 ).
\EEAS

We then can use the following simple lemma on sequences\footnote{The proof is by induction. This is true for $k=0$.   If this is true for $k-1$, then either (a) $u_{k-1} \geqslant \frac{B^2}{A^2} v_{k-1}$, then $u_k \leqslant u_{k-1} \leqslant \frac{B^2}{A^2} v_{k}^2 + B  v_{k}$ because $v_{k-1} \leqslant v_k$, or  (b) $u_{k-1} \leqslant \frac{B^2}{A^2} v_{k-1}$ and then $u_k \leqslant u_{k-1} + B v_{k-1} \leqslant  \frac{B^2}{A^2} v_{k-1}^2 + B  v_{k-1} \leqslant  \frac{B^2}{A^2} v_{k}^2 + B  v_{k}$. }: if $u_k \geqslant 0$ such that $u_0 = 0$ and $u_{k} \leqslant u_{k-1} - A \sqrt{u_{k-1}} + B v_{k-1}$ for $(v_k)$ non-decreasing and positive, then $u_k \leqslant  \frac{B^2}{A^2} v_k^2 + B  v_k  $.

This leads to a bound in $O(1)$ for $k^2 \| \Pi x_{k} - x_\ast \|^2$  for $\rho_k = 2/(k+1)$, and in $O((\log k)^2)$ for     $\rho_k =1/k$.

Note that this is similar to the proof of the convergence of Frank-Wolfe algorithms to an interior point of the feasible set, for which we have a rate of convergence of $f(x_{k}) - f(x_\ast) =   O(1/k^2)$~\citep{chen2010super}.

\paragraph{Putting things together.}
We then have for $\rho_{k} = \frac{1}{k}$,
$$
x_k = x_\ast +  \frac{k_0}{k}  (\idm - \Pi )  x_{k_0} + O((\log k)^2/k^2),
$$
which leads to $f(x_k) - f(x_\ast) = \frac{k_0}{k}   f'(x_\ast)^\top  (\idm - \Pi )  x_{k_0} + O((\log k)^2/k^2)$, which can be put back into the original bound, to obtain the bound without the logarithmic factor (since we have now replaced $O(\log (k) / k)$ by $ O((\log k)^2/k^2) = o(1/k)$ in the start of the proof). The dependence in $k$ can here lead to an acceleration.

For $\rho_{k} = \frac{2}{k+1}$, 
$$
x_k = x_\ast +   \frac{k_0 (k_0 + 1)}{k(k+1)}  (\idm - \Pi )  x_{k_0} + O(1/k^2),
$$
which leads to $f(x_k) - f(x_\ast) =  \frac{k_0 (k_0 + 1)}{k(k+1)}  f'(x_\ast)^\top  (\idm - \Pi )  x_{k_0} + O(1/k^2)$. The dependence in $k$ does not lead to an acceleration.

Note that the terms in $O(1/k^2)$ are not amenable to Richardson extrapolation because they are oscillating.

\section{Proofs of Prop.~\ref{prop:nesterov} and  Prop.~\ref{prop:nesterov2} (Nesterov smoothing)}
\label{app:smoothing}

We denote by $x_\lambda$ the minimizer of $h(x) + g_\lambda(x)$, and $\eta_\lambda$ the corresponding dual variable. The dual variable $\eta_\lambda$ is unique because of the strong-convexity of $\varphi$, while the primal variable is unique due to the same reasoning as in Appendix~\ref{app:preliminary} (when $\lambda$ tends to zero, $x_\lambda$ has to be close to  $x_\ast$, and in the neighborhood of $x_\ast$, $h$ is strongly-convex).

The primal problem is
$$
\min_{x \in \rb^d} h(x) + \lambda \varphi^\ast\Big( \frac{Ax- b}{\lambda} \Big),
$$
while the dual problem is 
$$
\max_{\eta \in\Delta_m} - \lambda \varphi(\eta) - h^\ast(-A^\top \eta) - \eta^\top b.
$$
The primal and dual solutions $x_\lambda$ and $\eta_\lambda$ are related through duality for $\varphi$, that is,
$$
\eta_\lambda = \partial \varphi^\ast \Big( \frac{Ax_\lambda- b}{\lambda} \Big),
$$
and through duality for $h$, that is,
$$
x_\lambda = \partial h^\ast( - A^\top \eta_\lambda).
$$
Since we consider functions $\varphi$ which are uniformly bounded, then we know that $f(x_\lambda) -f(x_\ast) = O(\lambda)$ and thus because the function is locally strongly convex, we have: $x_\lambda - x_\ast = O(\lambda^{1/2})$.

\subsection{Quadratic penalty}
With a quadratic penalty, and small enough $\lambda$, the solution $\eta_\lambda$ will have the same sparsity pattern, using standard techniques from active set methods~\citep{nocedal2006numerical}---that is, show that the dual solution constrained to the same active set leads to a globally optimal primal/dual solution. 

Moreover, because, once restricted to $I$, the dual function is locally strongly convex, and because $\varphi$ is bounded, the deviation in dual function values is less than $O(\lambda)$ and thus
$\eta_\lambda - \eta_\ast = O(\lambda^{1/2})$ (which is a bound we are going to improve below).

The optimality conditions for the dual problem become (stationarity with respect to $\eta_I$):
$$
0 = - \lambda (\eta_\lambda)_I + A_I \partial h^\ast( - A_I^\top (\eta_\lambda)_I) -b_I.
$$
Since $h$ is twice differentiable at $x_\ast$, $h''(x_\ast)$ is invertible,  and $x_\ast \in \partial h^\ast(-A^\top \eta_\ast)$, by the implicit function theorem, $h^\ast$ is twice differentiable at $-A^\top \eta_\ast$, and its Hessian is $h''(x_\ast)^{-1}$. We can thus further expand the optimality condition above as
$$ 0 
=  - \lambda (\eta_\lambda)_I + A_I \partial h^\ast( - A_I^\top (\eta_\ast)_I)
-  A_I \partial^2 h^\ast( - A_I^\top (\eta_\ast)_I) A_I^\top (  (\eta_\lambda)_I -  (\eta_\ast)_I ) - b_I
+ O ( \| \eta_\lambda - \eta_\ast \|^2).
$$
Since $\partial^2 h^\ast( - A_I^\top (\eta_\ast)_I)  = h''(x_\ast)^{-1}$
and $b_I =  A_I \partial h^\ast( - A_I^\top (\eta_\ast)_I)$ (because of optimality of $x_\ast$ and $\eta_\ast$), this leads to
$0 =  - \lambda (\eta_\lambda)_I  + A_I h''(x_\ast)^{-1} A_I^\top (  (\eta_\lambda)_I -  (\eta_\ast)_I )   + O ( \| \eta_\lambda - \eta_\ast \|^2)$. Since we know already that  $\eta_\lambda - \eta_\ast = O(\lambda^{1/2})$, this leads to $\eta_\lambda - \eta_\ast =  O(\lambda)$,
which in turn leads to
$$
 (\eta_\lambda)_I =   (\eta_\ast)_I - \lambda \big[ A_I h''(x_\ast)^{-1} A_I^\top  \big]^{-1} (\eta_\ast)_I  + O(\lambda^2),
$$
which is the desired expansion for the dual variable. We then get:
$$
x_\lambda = \partial h^\ast( - A^\top \eta_\lambda)
= x_\ast + \lambda h''(x_\ast)^{-1} A_I^\top  \big[ A_I h''(x_\ast)^{-1} A_I^\top  \big]^{-1} (\eta_\ast)_I + O(\lambda^2)
= x_\ast + \lambda \Delta + O(\lambda^2).
$$
Thus, $2x_{\lambda} - x_{2 \lambda} = O(\lambda^2)$, and
\BEAS
h(x_{\lambda}) + g( x_{\lambda}) & = & h(x_\ast)+g(x_\ast) + \lambda h'(x_\ast)^\top \Delta + g(x_\ast + \lambda \Delta) - g(x_\ast)
+ O(\lambda^2)  \\
& = &  h(x_\ast)+g(x_\ast) + \big[ g(x_\ast + \lambda \Delta) - g(x_\ast)
 + \lambda h'(x_\ast)^\top \Delta  \big] + O(\lambda^2 ).
\EEAS

Since by optimality $h'(x_\ast) \in - \partial g(x_\ast)$, the term $\big[ g(x_\ast + \lambda \Delta) - g(x_\ast)
 + \lambda h'(x_\ast)^\top \Delta  \big]$ resembles a Taylor expansion of $g$ at $x_\ast$, but in general, we cannot have a term in $O(\lambda^2)$ because of the non-smoothness of $g$.
 For example, for the $\ell_1$-norm, we get
  $\big[ g(x_\ast + \lambda \Delta) - g(x_\ast)
 + \lambda f'(x_\ast)^\top \Delta  \big] = \lambda \| \Delta_{I^c} \|_1 +  \lambda f'(x_\ast)^\top \Delta$, is not zero, and only $O(\lambda)$.
 
 For Richardson extrapolation, we get:
\BEAS
h(2x_{\lambda} - x_{2 \lambda}) + g(2x_{\lambda} - x_{2 \lambda}) & = & h(x_\ast)+g(x_\ast) + O(\lambda^2),
\EEAS
and thus an improvement from $O(\lambda)$ to $O(\lambda^2)$.

\subsection{Entropic penalty}
 
For $\varphi(\eta) = \sum_{i=1}^m \big\{ \eta_i \log \eta_i - \eta_i \big\}$, we have
$\varphi'(\eta)_i = \log \eta_i$, and we cannot use anymore thet fact that $\eta_\lambda$ has the same sparsity pattern as $\eta_\ast$, since all components of $\eta_\lambda$ are non zero. However, since the entropy is bounded over the simplex, we still have 
$\eta_\lambda - \eta_\ast = O(\lambda^{1/2})$, and from the same reasoning as for the quadratic penalty, $x_\lambda - x_\ast = O(\lambda^{1/2})$.

Thus, by writing primal-dual optimality conditions, we get:
\BEAS
& & -\lambda \log \eta_\lambda + Ax_ \lambda- b = 0 \\
& & h'(x_\lambda) + A^\top \eta_\lambda = 0.
\EEAS
This implies that for $i \notin I$,
$$
\log \eta_i \sim (Ax_\ast-b)_i - \sup_{j} (Ax_\ast-b_j),
$$
which is strictly negative by assumption. Thus $(\eta_{I^c})_\lambda = O( \rho^{-1/\lambda})$ for a certain $\rho \in (0,1)$.
We   now show that $(\eta_I)_\lambda = (\eta_\ast)_I + \lambda \Delta + O(\lambda^2)$.
From the optimality conditions, we get:
$$
0 = - \lambda   (\log(\eta)_\lambda)_I + A_I \partial h^\ast( - A^\top \eta_\lambda ) -b_I,
$$
which leads to, with an additional Taylor expansion,
$$0 =   - \lambda (\log(\eta)_\ast)_I + A_I \partial h^\ast( - A_I^\top (\eta_\ast)_I)
-  A_I \partial^2 h^\ast( - A_I^\top (\eta_\ast)_I) A^\top (  \eta_\lambda -  \eta_\ast ) - b_I + O(\|\eta_\lambda - \eta_\ast\|^2)
+ O(\lambda \|\eta_\lambda - \eta_\ast\|).
$$
Since $\partial^2 h^\ast( - A_I^\top (\eta_\ast)_I)  = h''(x_\ast)^{-1}$
and $b_I =  A_I \partial h^\ast( - A_I^\top (\eta_\ast)_I)$ (because of optimality of $x_\ast$ and $\eta_\ast$), with moreover  $(\eta_{I^c})_\lambda = O( \rho^{-1/\lambda})$,
this leads to
$$ - \lambda (\log(\eta)_\ast)_I  - A_I h''(x_\ast)^{-1} A_I^\top (  (\eta_\lambda)_I -  (\eta_\ast)_I  )  = O(\| (\eta_\lambda)_I -  (\eta_\ast)_I\|^2)
+ O(\lambda \| (\eta_\lambda)_I -  (\eta_\ast)_I\|) + O(\rho^{-1/\lambda}),$$ which in turn leads to
$$
 (\eta_\lambda)_I =   (\eta_\ast)_I -  \lambda \big[ A_I f''(x_\ast)^{-1} A_I^\top  \big]^{-1} (\log(\eta_\ast))_I  + O(\lambda^2),
$$
which is the desired expansion for the dual variable. We then get:
$$
x_\lambda = \partial h^\ast( - A^\top \eta_\lambda)
= x_\ast + \lambda h''(x_\ast)^{-1} A_I^\top  \big[ A_I h''(x_\ast)^{-1} A_I^\top  \big]^{-1} (\log(\eta_\ast))_I + O(\lambda^2)
= x_\ast + \lambda \Delta + O(\lambda^2).
$$

Note that following~\citet{cominetti1994stable}, we could get a single proof for all $\varphi(\eta) = \sum_{i=1}^n \psi(\eta_i)$, with replacing $\log \eta_\ast$ by $\psi'(\eta_\ast)$, with a condition to ensure that the zero variables lead to vanishing terms.

\subsection{Higher-order expansions (Prop.~\ref{prop:nesterov2})}
In order to use $m$-step Richardson extrapolation, we need to have a bound of the form
$$
x_\lambda = x_\ast + \sum_{i=1}^m \Delta_i \lambda_i + O(\lambda^{m+1}).
$$
We only consider for simplicity the quadratic penalty, where we simply need an expansion of $(\eta_\lambda)_I$ (minor modifications would lead to a proof for the entropic penalty since there $\eta_{I^c}$ is exponentially small).

The expansion can be obtained from the implicit function theorem applied to 
the equation $0 = - \lambda (\eta_\lambda)_I + A_I \partial h^\ast( - A_I^\top (\eta_\lambda)_I) -b_I$, which is of the form $H( (\eta_\lambda)_I,\lambda) = 0$, where $H$ has high-order derivatives as long as $h^\ast$ is sufficiently differentiable, and the partial derivative with respect to $(\eta_\lambda)_I$ is an invertible matrix. The high-order differentiability of $h^\ast$ around $-A^\top \eta_\ast$ comes from the implicit function theorem, applied to the definition of the gradient of the Fenchel conjugate.

\begin{figure}
\begin{center}
\hspace*{-.2cm}
\includegraphics[scale=.41]{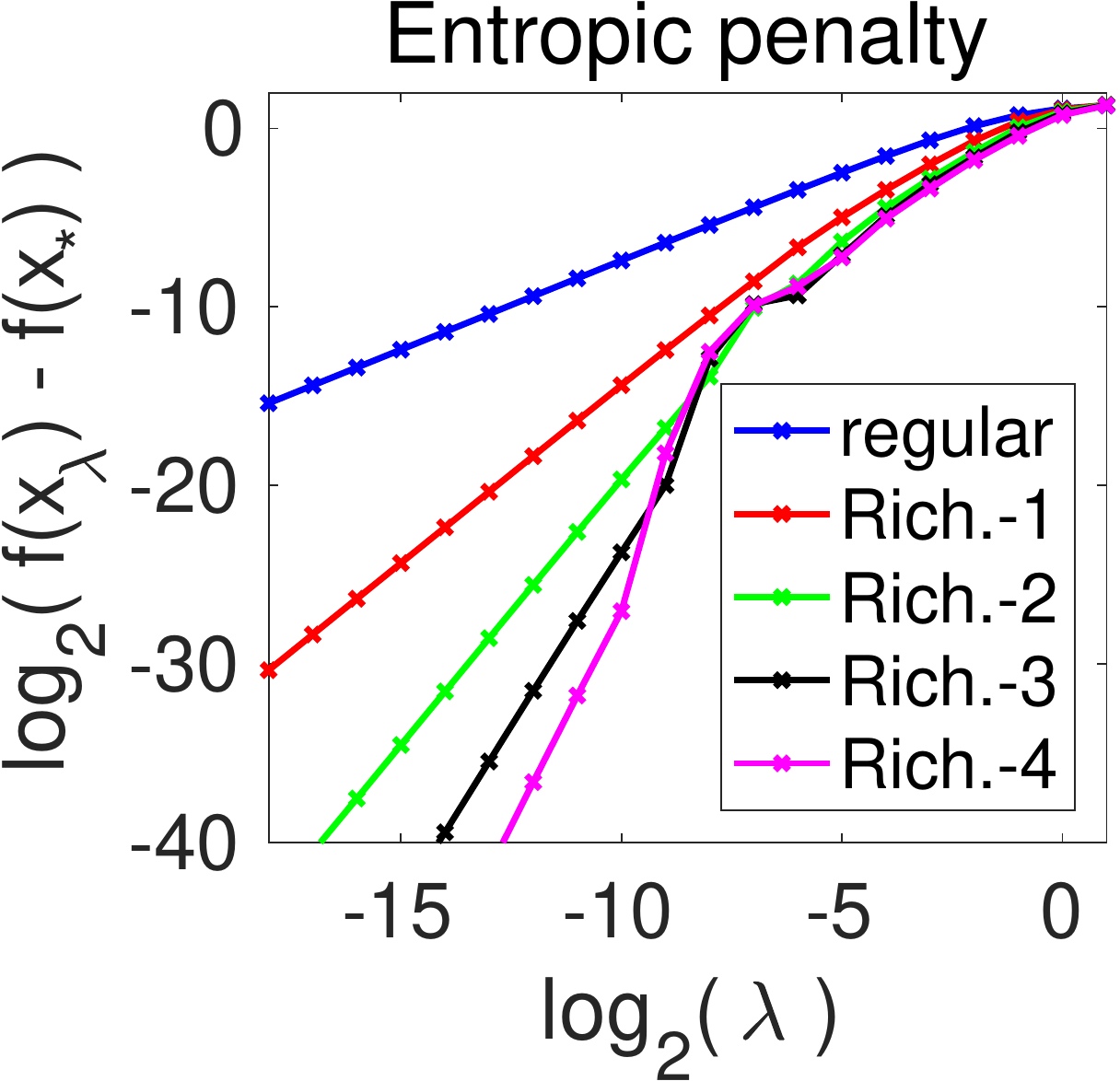} \hspace*{.4cm}
\includegraphics[scale=.41]{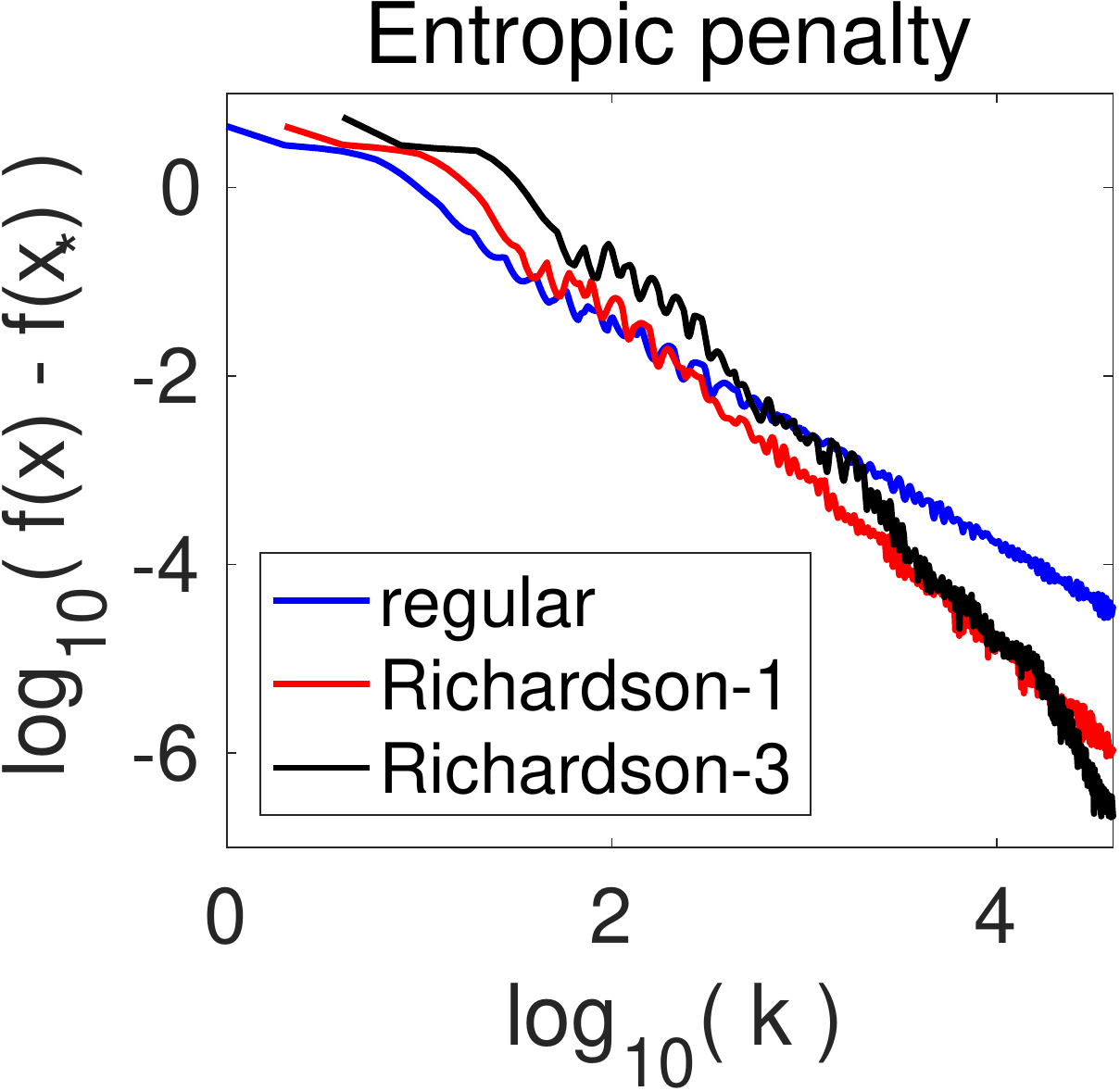} 

\vspace*{-.2cm}

\caption{Richardson extrapolation for Nesterov smoothing on a penalized Lasso problem, with regularization by the entropic penalty.  Left:  dependence of $f(x_\lambda) - f(x_\ast)$ on $\lambda$, for Richardson extrapolation of order $m$, we indeed recover a slope of $m+1$ in the log-log plot. Right: Optimization error vs.~number of iterations; where we go from a slope of -1 (blue curves) to improved slopes of $-4/3$ (red curve) and $-8/5$ (black curve). See text for details.\label{fig:smoothing_ent}}
\end{center}
\end{figure}

\section{Ridge regression}
\label{app:ridge}

\subsection{Standard extrapolation}

We have $\hat{y}_\lambda=K ( K + n\lambda \idm)^{-1}  y = \hat{H}_\lambda y$, and thus
$2 \hat{y}_\lambda - \hat{y}_{2\lambda}  = [ (2 \hat{H}_\lambda - \hat{H}_{2\lambda} ) ] y$, and we can compute explicitly
\BEAS
2 \hat{H}_\lambda - \hat{H}_{2\lambda} - \idm & = &   2 K ( K + n\lambda \idm)^{-1} -  K ( K + 2 n\lambda \idm)^{-1} -  \idm  \\
& = & 2n \lambda  \big[  ( K + 2n\lambda \idm)^{-1} - ( K + n\lambda \idm)^{-1} \big] = - 2 (n \lambda)^2 ( K + n\lambda \idm)^{-1} ( K + 2n\lambda \idm)^{-1} .
\EEAS
Thus,
\BEAS
{\rm bias}(2 \hat{H}_\lambda - \hat{H}_{2\lambda} ) & \leqslant &4 n^3 \lambda^4 z^\top
( K + n\lambda \idm)^{-4} z \\
{\rm variance }(2 \hat{H}_\lambda - \hat{H}_{2\lambda} ) & \leqslant & \frac{\sigma^2}{n} \big[
 4\tr [K ( K + n\lambda \idm)^{-1}]^2  +  \tr[ K ( K + 2n\lambda \idm)^{-1} ]^2
 \big] \leqslant \frac{5\sigma^2}{n}  \tr [K ( K + n\lambda \idm)^{-1}]^2  .
\EEAS
We have ${\rm bias}(2 \hat{H}_\lambda - \hat{H}_{2\lambda} ) \leqslant 4 {\rm bias}(  \hat{H}_\lambda  ) $
and ${\rm variance}(2 \hat{H}_\lambda - \hat{H}_{2\lambda} ) \leqslant 5 {\rm variance}(  \hat{H}_\lambda  ) $, so the two terms never incur more than a constant factor.

However, the bias can be much improved. Following~\citet[Section 4.3]{bach2013sharp}, if the eigenvalues of $K$ are $\Theta( n \mu_i)$ and the coordinates of $z$ in the eigenbasis of $K$ are $\Theta( \sqrt{n  \nu_i})$, we can compute equivalents (up to constant terms) of the bias and variance terms, for different types of decays. See Table~\ref{tab:decays} with $m=0$. Since the optimal predictive performance is $(\frac{\sigma^2}{n})^{1-1/(2 \delta) }$, the only potential gains to go from $m=0$ (no extrapolation)  to $m>0$ (extrapolation) occur when $(\nu_i)$ has a fast decay (that is, first, third and fifth lines).

For the first line in Table~\ref{tab:decays}, we will show that the bias term for Richardson extrapolation is of the order $\lambda^4$ if $2\delta \!> \!8 \beta \!+\! 1$, and equal to $ \lambda^{(2 \delta -1)/2 \beta}$ when 
$2\delta \!<  \!8 \beta \!+\! 1$. We will thus increase the regime of validity of the bias term that leads to optimal performance, from 
$2\delta \! <  \!4 \beta \!+\! 1$ to $2\delta \! <  \!8 \beta \!+\! 1$. More generally, as we show below in Appendix~\ref{app:ridgemulti}, the bias for $m$-step extrapolation is 
proportional to $n^{2m+1} \lambda^{2m+2} z^\top
( K + n\lambda \idm)^{-2m-2} z$, and we bound it directly following closely the computations of~\citet[Appendix C.2]{bach2013sharp}:
\BEAS
n^{2m+1} \lambda^{2m+2} z^\top
( K + n\lambda \idm)^{-2m-2} z  & = & n^{2m+2}  \lambda^{2m+2} \sum_{i=1}^n \frac{\nu_i}{ (n\mu_i + n\lambda)^{2m+2}} = 
 \lambda^{2m+2} \sum_{i=1}^n \frac{\nu_i}{ (\mu_i + \lambda)^{2m+2}} \\
& = & \lambda^{2m+2}\sum_{i=1}^n \frac{i^{-2\delta}}{ (i^{-2\beta} + \lambda)^4} \leqslant 
2  \lambda^{2m+2} \int_1^n \frac{t^{-2\delta}}{ (t^{-2\beta} + \lambda)^{2m+2}} dt \\
& = &  2 \lambda^{2m+2} \int_1^n \frac{t^{4(m+1)\beta - 2\delta}}{ (1 + \lambda t^{2\beta} )^{2m+2}} dt.
 \EEAS
If $2\delta - 4(m+1) \beta > 1$, then we have an upper bound of $\displaystyle 2 \lambda^{2m+2} \int_1^n  {t^{4(m+1) \beta - 2\delta}}dt = O(\lambda^{2m+2})$.

If $2\delta - 4(m+1)  \beta < 1$, then we can further bound
\BEAS
2 \lambda^{2m+2}\int_1^n \frac{t^{4(m+1)\beta  - 2\delta}}{ (1 + \lambda t^{2\beta} )^{2m+2}} dt & = & 
2 \lambda^{2m+2} \int_\lambda^{\lambda n^{2\beta}} \frac{[(u/\lambda)^{1/2\beta}]^{4(m+1)\beta - 2\delta+1 }}{ (1 + u )^{2m+2}} \frac{1}{2\beta}du
\\
 & & \mbox{ with the change of variable } u = \lambda t^{2\beta}, \\
 & = & 
2 \lambda^{2m+2 -(2m+2) + \delta /\beta - 1/(2\beta)} \int_\lambda^{\lambda n^{2\beta}} \frac{u^{(4(m+1)\beta - 2\delta + 1) ) /((4(m+1)\beta) }}{ (1 + u )^{2m+2}} \frac{1}{2\beta}du \\
& = &  O ( \lambda^{(2\delta-1)/(2\beta)}),
\EEAS
because the integral in convergent. The bias term for the third and fifth lines of  Table~\ref{tab:decays} needs modifications in exactly the same way the corresponding proof from~\citet[Appendix C.2]{bach2013sharp}.

In order to find the optimal regularization parameter, we minimize with respect to $\lambda$, that leads to
$ \lambda^{2(m+1)+1/2\beta} \propto ( \sigma^2 / n)$, leading to an optimal prediction performance proportional
to $(\sigma^2/n)^{\tau}$, with 
$\tau = \frac{2(m+1)}{2(m+1)+\frac{1}{2\beta}} = 1 -  \frac{\frac{1}{2\beta} }{2(m+1)+\frac{1}{2\beta}}
= 1 -  \frac{1 }{4(m+1)\beta +1} $.

\begin{table}

\centering
\hspace*{-.5cm}
\begin{tabular}{|l|l|l|l|l|l|l|}
\hline
$(\mu_i)$& $(\nu_i)$ & var.& bias & optimal $\lambda$ & pred. perf. &  condition    \\[.05cm]
\hline
\\[-.7cm]
& & & & &  &   \\
$  i^{-2\beta} $  & $ i^{-2\delta} $  & $  \frac{\sigma^2}{n} \lambda^{-1/2\beta}$ & $  \lambda^{2\textcolor{red}{(m+1)}}$  &  
$ (\frac{\sigma^2}{n})^{1/(2\textcolor{red}{(m+1)}+ 1/2\beta)}$ & $(\frac{\sigma^2}{n})^{1-1/(4 \textcolor{red}{(m+1)}\beta+ 1)}$ &  
if $2\delta \!> \!4\textcolor{red}{(m+1)} \beta \!+\! 1$\\[.1cm]
$  i^{-2\beta} $  & $ i^{-2\delta} $  & $   \frac{\sigma^2}{n} \lambda^{-1/2\beta}$ & $ \lambda^{(2 \delta -1)/2 \beta}$  & $(\frac{\sigma^2}{n})^{\beta/\delta}$ &  $(\frac{\sigma^2}{n})^{1-1/(2 \delta) }$ &   
if $2\delta \! <\! 4 \textcolor{red}{(m+1)}\beta\! +\! 1$\\[.1cm]
$   i^{-2\beta} $  & $e^{-\kappa i}$  & $  \frac{\sigma^2}{n} \lambda^{-1/2\beta}$ & $\lambda^{2\textcolor{red}{(m+1)}}$ & $ (\frac{\sigma^2}{n})^{1/(2\textcolor{red}{(m+1)} + 1/2 \beta)}$ & $(\frac{\sigma^2}{n})^{1-1/(4 \textcolor{red}{(m+1)}\beta+ 1)}$ &   \\[.1cm]
$ e^{-\rho i} $  & $  i^{-2\delta} $  &   $ \frac{\sigma^2}{n} \log \frac{1}{\lambda}$  & $  ( \log \frac{1}{\lambda})^{1-2\delta}$&$ \exp( - (\frac{\sigma^2}{n})^{-1/(2\delta)} )$  &$(\frac{\sigma^2}{n})^{1-1/(2 \delta) }$     & \\[.1cm]
$  e^{-\rho i} $  & $e^{-\kappa i}$  & $ \frac{\sigma^2}{n} \log \frac{1}{\lambda}$ &  $ \lambda^{2\textcolor{red}{(m+1)}}$  &  $(\frac{\sigma^2}{n})^{1/2}$& $(\frac{\sigma^2}{n}) \log (\frac{n}{\sigma^2}) $      & if $\kappa > 2\rho$ \\[.1cm]
$  e^{-\rho i} $  & $e^{-\kappa i}$  & $ \frac{\sigma^2}{n} \log \frac{1}{\lambda}$ & $\lambda^{\kappa/\rho} $ & $(\frac{\sigma^2}{n})^{\rho/\kappa}$ & $(\frac{\sigma^2}{n}) \log (\frac{n}{\sigma^2}) $   & if $\kappa <2 \rho$ \\[.1cm]
\hline
\end{tabular}

\caption{
Variance, bias, optimal regularization parameter, corresponding prediction performance, for several decays of eigenvalues and signal coefficients, and $m$-th order Richardson extrapolation (we always assume $\delta>1/2$, $\beta>1/2$, $\rho>0$, $\kappa>0$, to make the series summable). All entries are functions of $i$, $n$ or $\lambda$ and are only asymptotically bounded below and above, i.e., correspond to the asymptotic notation $\Theta(\cdot)$. Adapted from~\citet{bach2013sharp}, changes due to potential Richardson extrapolation are highlighted \textcolor{red}{in red}.    }
\label{tab:decays}

 \vspace*{-.2cm}

\end{table}

\subsection{Multiple extrapolation steps}
\label{app:ridgemulti}

Using $m$ steps of Richardson interpolation, we prove that we can get this regime as long as
$2\delta \! <  \! 4(m+1)\beta \!+\! 1$, and thus with no limit if $m$ is large enough. 

We thus consider
\BEAS
\hat{H}^{(m)}_\lambda & = &  \sum_{i=1}^{m+1}\alpha^{(m)}_i \hat{H}_{i \lambda} 
=  \sum_{i=1}^{m+1} \alpha^{(m)}_i K ( K + n\lambda i \idm)^{-1} 
= \idm  -n \lambda \sum_{i=1}^{m+1} i\alpha^{(m)}_i   ( K + n\lambda i \idm)^{-1} ,
\EEAS
using $\sum_{i=1}^{m+1}   \alpha^{(m)}_i = 1$. We then use
\BEAS
& & ( K + n\lambda i \idm)^{-1}  
\\ & = & 
 ( K + n \lambda \idm)^{-1/2} 
 \big[ \idm + n \lambda (i-1) ( K + \lambda \idm)^{-1}   \big]^{-1}
  ( K + n \lambda \idm)^{-1/2} 
\\
 & = & ( K + n \lambda \idm)^{-1/2} 
 \sum_{j=0}^{m-1} (-1)^j \big[  n \lambda  ( K + n\lambda(i-1) \idm)^{-1}    \big]^j
  ( K + n\lambda \idm)^{-1/2}   \\
  & & \hspace*{2cm} + (-1)^{m} ( K + n \lambda \idm)^{-1/2}  \big[  n \lambda(i-1)  ( K + n\lambda \idm)^{-1}    \big]^{m}  \big[ \idm + n \lambda  ( K + \lambda(i-1) \idm)^{-1}  \big]^{-1}  ( K + n \lambda \idm)^{-1/2}
\\
 & = & 
 \sum_{j=0}^{m-1} (-1)^j  ( n \lambda)^j  ( K + n\lambda \idm)^{-j-1}(i-1)^j  
 +  (-1)^{m}  ( n \lambda)^{m}  ( K + n\lambda \idm)^{-m} (i-1)^{m}  ( K + n\lambda i \idm)^{-1}  
 ,
\EEAS
where we have used the sum of a geometric series, with $A = n \lambda(i-1)  ( K + \lambda \idm)^{-1} $:
\BEAS
(I+A)^{-1} & = & \sum_{j=0}^{m-1} (-1)^j A^j + (-1)^{m} (I+A)^{-1} A^{m}.
\EEAS
Putting things together, we get
\BEAS
\hat{H}^{(m)}_\lambda
- \idm
& = & 
-n \lambda \sum_{i=1}^{m+1} i \alpha^{(m)}_i   \sum_{j=0}^{m-1} (-1)^j  ( n \lambda)^j  ( K + n\lambda \idm)^{-j-1} (i-1)^j   \\
& & -n \lambda \sum_{i=1}^{m+1} i \alpha^{(m)}_i (-1)^{m}  ( n \lambda)^{m}  ( K + n\lambda \idm)^{-m} (i-1)^{m}  ( K + n\lambda i \idm)^{-1}  .
\EEAS
The first term is exactly zero by definition of the Richardson weights $\alpha_i^{(m)}$, and we are left with 
\BEAS
 (-1)^{m+1} \big[ \hat{H}^{(m)}_\lambda
- \idm \big]
& = &   ( n \lambda)^{m+1}  \sum_{i=1}^{m+1} i \alpha^{(m)}_i (i-1)^{m}  ( K + n\lambda \idm)^{-m}( K + n\lambda i \idm)^{-1}  .
\EEAS
Thus
$$( n \lambda)^{-m-1}   ( K + n\lambda \idm)^{-m/2-1/2} \big[ \hat{H}^{(m)}_\lambda
- \idm \big] ( K + n\lambda \idm)^{-m/2-1/2}
$$
has an operator norm bounded by a constant that depends on $m$ (and not on other quantities like $\lambda$ or $n$). This allows to bound the bias  as
\BEAS
{\rm bias}(\hat{H}^{(m)}_\lambda ) & = &  \frac{1}{n} \| (\hat{H}^{(m)}_\lambda
- \idm ) z\|^2 \leqslant \square_m n^{2m+1} \lambda^{2m+2} z^\top
( K + n\lambda \idm)^{-2m-2} z .
\EEAS
Similarly to the case $m=1$, $(\hat{H}_\lambda^{(m)} )^2$ is upper-bounded by a sum of matrices which are all less than $K^2(K+n\lambda \idm)^{-2}$ (for the order between symmetric matrices), leading to a bound on the variance term as:
\BEAS
{\rm variance }(\hat{H}^{(m)}_\lambda) & =  & \sigma^2 \tr  [\hat{H}_\lambda^{(m)} ]^2 \leqslant \triangle_m \frac{\sigma^2}{n}  \tr [K ( K + n\lambda \idm)^{-1}]^2 .
\EEAS
Here $\square_m$ and $\triangle_m$ are constants that could be explicitly computed. As shown in the section, these constants have to diverge when $m$ tends to $+\infty$.

\subsection{Explicit expression}
\label{app:gamma}

\paragraph{Expression for $\alpha_i^{(m)}$.} We first give a proof for the explicit expression for $\alpha_i^{(m)}$. One approach is to solve Vandermonde matrices like done by~\citet{pages2007multi} in a similar context, but given the conjecture, we can simply check that it satisfies \eq{alpha1} and \eq{alpham}.

For \eq{alpha1}, we have:
\BEAS
\sum_{i=1}^{m+1} (-1)^{i-1} { m+1 \choose i} & = & 1 - \sum_{i=0}^{m+1} (-1)^{i} { m+1 \choose i} = 0,
\EEAS
using the binomial formula.
For \eq{alpham}, we have:
\BEAS
\sum_{i=1}^{m+1} (-1)^{i-1} { m+1 \choose i} i & = & \sum_{i=1}^{m+1} (-1)^{i-1} (m+1) { m \choose i-1} = (m+1) \times 0 = 0,
\EEAS
and, more generally for any $j \in \{1,\dots,m\}$,
\BEAS
\sum_{i=j}^{m+1} (-1)^{i-1} { m+1 \choose i} i(i-1) \cdots (i-j+1) & = & \sum_{i=j}^{m+1} (-1)^{i-1} (m+1)m \cdots (m-j+2) { m-j+1 \choose i-j} =  0,
\EEAS
also using the binomial formula. This shows that $\sum_{i=j}^{m+1} (-1)^{i-1} { m+1 \choose i} i^j = 0$ for all $j \in \{1,\dots,m\}$, which finishes the proof of the formula for $\alpha_{i}^{(m)}$.

\paragraph{Expression for the smoothing matrix.} We can now provide an explicit expression for the extrapolated smoothing matrix. We have:
\BEAS
\hat{H}^{(m)}_\lambda & = &  \sum_{i=1}^{m+1}\alpha^{(m)}_i \hat{H}_{i \lambda} \\
& = &   \sum_{i=1}^{m+1} \alpha^{(m)}_i K ( K + n\lambda i \idm)^{-1}  = \idm  -n \lambda \sum_{i=1}^{m+1} i\alpha^{(m)}_i   ( K + n\lambda i \idm)^{-1}  \\
& = & \idm  -n \lambda \sum_{i=1}^{m+1} i (-1)^{i-1} { m+1 \choose i}    ( K + n\lambda i \idm)^{-1} = s(K/(n \lambda)),
\EEAS
where $s: \rb^{n \times n } \to \rb^{n \times n}$ is a spectral function defined on symmetric matrices from a function (note the classical overloaded notation) $s: \rb \to \rb$, by keeping eigenvectors unchanged and applying $s$ to eigenvalues~\citep[][Chapter 11]{golub1989matriz}. We have, using a representation by an integral and the binomial formula:
\BEAS
s(\mu) & = & 
1  -  \sum_{i=1}^{m+1} i (-1)^{i-1} { m+1 \choose i}    \frac{1}{
 \mu  +  i } =  
1  -  \sum_{i=1}^{m+1} (m+1) (-1)^{i-1}{ m \choose i-1}    
\int_0^1 t^{ \mu +  i - 1} dt\\
 & = & 
1  -  (m+1) \int_0^1  t^{ \mu }  \Big[  \sum_{i=1}^{m+1} (-1)^{i-1} { m \choose i-1}    
t^{   i - 1} \Big] dt\\
 & = & 
1  -  (m+1) \int_0^1  t^{ \mu }  (1-t)^mdt = 
1  -  (m+1)  \frac{\Gamma(1+\mu ) \Gamma(1+m)}{\Gamma(2 + m  +\mu ) },
\EEAS
using the expression of the Beta function in term of the Gamma function $\Gamma$~\citep{abramowitz1988handbook}. We can simply the expression as follows:
\BEAS
s(\mu) & = & 1 - \frac{ (m+1)! }{   (\mu+1)(\mu+2)\cdots (\mu+m+1)}.
\EEAS
This provides a new closed-form expression for Richardson extrapolation, as well as it provides an equivalent when $m$ tends to $+\infty$, as $s(\mu) \sim 1 - \frac{ \Gamma(1+\mu)} {m^\mu}$. Therefore, when $m \to +\infty$, and $\mu>0$, then $s(\mu)$ tends to one (which implies that the constants $\square_m$ and $\triangle_m$ cannot remain bounded). Therefore, the variance term converges to $\sigma^2$,  but very slowly: the method does not blow up, but does not learn either.

\end{document}